\documentclass[twoside,british,3p]{elsarticle}
\usepackage[T1]{fontenc}
\usepackage[latin9]{inputenc}
\usepackage[table,xcdraw]{xcolor}
\usepackage{pagecolor}
\usepackage{pdfcolmk}
\usepackage{units}
\usepackage{textcomp}
\usepackage{bbding}
\usepackage{amsmath}
\usepackage{amssymb}
\usepackage{graphicx}
\usepackage{tikz}
\usetikzlibrary{arrows.meta,positioning,shapes.geometric}
\usepackage{setspace}
\usepackage{esvect}
\PassOptionsToPackage{normalem}{ulem}
\usepackage{ulem}
\usepackage{blindtext}
\usepackage{multicol}
\usepackage{booktabs}
\usepackage{tabularx}
\usepackage{longtable}
\usepackage{pbox}
\usepackage{csquotes}

\usepackage{caption}
\usepackage{subcaption}

\usepackage{lscape}
\usepackage{hyperref}
\usepackage{makecell}
\usepackage{multirow}

\usepackage{enumitem}

\usepackage{natbib}

\newtheorem{theorem}{Observation}
\newcount\savedtheorem

\hypersetup{
    colorlinks=true,
    linkcolor=blue,
    filecolor=magenta,
    urlcolor=cyan,
    pdfborderstyle={/S/U}
}

\makeatletter

\journal{JORS}
\bibpunct{(}{)}{;}{a}{,}{,}

\DeclareGraphicsExtensions{.pdf,.jpeg,.png,.eps}

\usepackage{dcolumn}
\usepackage{bm}
\usepackage{babel}

\newcommand{\HThreeOrderLengths}{150}
\newcommand{\HThreeOrderLargestName}{Qwen3.8 27B}
\newcommand{\HThreeOrderSmallestName}{Qwen3 0.6B}
\newcommand{\HThreeOrderEnergyFwWithin}{136}
\newcommand{\HThreeOrderEnergyFwMedian}{0}

\newcommand{\HThreeOrderEnergyLargeMedian}{300}

\newcommand{\HThreeOrderEnergyLargeTimesAtTwenty}{11}

\newcommand{\HThreeOrderEnergySmallMedian}{250}
\newcommand{\HThreeOrderEnergySmallLastCompleted}{32}

\newcommand{\HThreeToolsEnergyFwWithin}{128}

\newcommand{\HThreeToolsEnergyLargeTimesAtTen}{10}

\newcommand{\HThreeToolsEnergySmallLastCompleted}{23}

\newcommand{\HThreeToolsTimeFwWithin}{94}

\begin{document}

\begin{frontmatter}{}

\title{When Is Smaller Greener? Right-Sizing Language Models for Agentic AI}

\author[aff1,aff2]{Sajid Siraj}
\ead{s.siraj@leeds.ac.uk}
\author[aff2]{Maria Barbati\corref{fn1}}
\ead{maria.barbati@unive.it}

\cortext[fn1]{Corresponding author}

\address[aff1]{Leeds University Business School, University of Leeds, Leeds, United Kingdom}
\address[aff2]{Ca'Foscari University of Venice, Venice, Italy}

\begin{abstract}
AI agents use a large language model (LLM) to plan and act over many steps. Which model should an agent use, so that its tasks are completed reliably with the least energy, cost and time? We treat this as a multi-criteria decision problem in which comparisons per call can mislead: a task succeeds only if all its steps succeed, and failed attempts still use resources. We therefore compare models on resources per successfully completed task. We derive crossover task lengths at which two models exchange their efficiency ranking, determined by their per-step reliability and resource use. Each objective has its own crossover. Below the shortest crossover the less reliable model dominates, above the longest the more reliable one dominates, and in between both are Pareto-efficient and preferences decide. This supports a two-stage framework: a preference-free dominance screen, followed by a multi-criteria ranking of the efficient models, subject to reliability and feasibility constraints. We also give conditions under which this structure holds when reliability and resource use change during a task. In a pilot with 15 open-weight models on two synthetic agent tasks, with measured GPU energy, the predicted crossover fell within the observed range for 72 of the 88 observed crossovers. At 20 steps, a 3-billion-parameter model used a tenth of the energy of a 27-billion one per completed task; at 100 steps, a 1.7-billion model used almost four times as much as an 8-billion one. Smaller is greener only up to a task length that can be predicted.
\end{abstract}


\end{frontmatter}

\section{Introduction\label{sec:intro}}

Large language models (LLMs) are no longer used only to answer questions in a chat window. They now act as the reasoning engine inside AI agents, which plan, call software tools, retrieve information and take actions over many steps \citep{lei2025agents,zhu2026orchestration}. An agent that handles a customer request, for example, may read a booking record, check a policy, call a payment system and send a confirmation, with the model deciding what to do at each step. At the same time, the number of available models has grown quickly. They range from very large models offered by a few providers to small open models that can run on modest hardware, and they differ widely in reliability, price, speed and energy use.

This paper asks one question: \emph{which language model should an AI agent use, so that its tasks are completed reliably with the least energy, cost and time?} We call this the \emph{selection question}, and it is the thread that runs through the paper. We return to it at the start of each section.

The selection question matters for three reasons. First, the choice is made very often. A deployed agent may run millions of tasks, and each task may involve dozens of model calls. Second, a wrong choice has real consequences. A model that is not reliable enough can send a wrong message or change a record by mistake \citep{gritta2026process,omar2026poisoned}, while a model that is more capable than needed wastes money, time and energy \citep{yehudai2025survey}. Third, the energy used by AI inference is now a public concern, and the share used by agents is growing \citep{oviedo2026energy,elsworth2025google}.

Current practice offers two common answers. The first is to use the most capable model, as judged by a public leaderboard. The second, often presented as the sustainable choice, is to use the smallest or cheapest model that seems good enough \citep{belcak2025slm}. Both answers compare models \emph{per call}: per token, per request or per benchmark question. The same is true of most work on routing requests between models \citep{chen2024frugalgpt,ong2025routellm} and of recent work on carbon-aware routing \citep{sheshanarayana2026gar}.

We argue that per-call comparisons can mislead when the user of the model is an agent. An agent completes a task only if every step of the task succeeds. A small difference in the reliability of each step therefore becomes a large difference in the success of the whole task, and a failed task must be repeated or handed to a person. A model that uses less energy per call may then use more energy per completed task. In other words, the smaller model is not always the greener choice, and whether it is depends on the task.

This paper makes four contributions to the selection question.
\begin{enumerate}[label=(\arabic*)]
  \item We take the \emph{successfully completed task} as the unit of comparison, in the sense of the functional unit of life-cycle assessment \citep{iso14044}, and build a general decision model on it. The model covers tasks with binary or graded outcomes and actions that can or cannot be undone.
  \item We derive simple results that say when a smaller model stops being the greener, cheaper or faster choice. Each resource has its own \emph{crossover} task length, and the crossovers divide tasks into three regimes. In two of them the answer needs no preference information at all. We give a condition under which this structure survives when the reliability and the resource use of a step change along the task, and show that it holds for nearly all pairs of models in our pilot.
  \item We build these results into a two-stage decision framework. A preference-free Pareto screen removes every model that is worse on all counts, and a weighted choice is made only where a real trade-off remains.
  \item We set out how the framework can be tested, and report a pilot study with 15 open-weight models from five developers on two kinds of agent task, with measured GPU energy. The pilot supports the predicted structure, and for most pairs of models the predicted crossover lengths agree with the observed ones.
\end{enumerate}

The paper speaks to three communities. For AI research, it gives a way to predict when small models are enough for agentic tasks. For operational research, it gives a formal model of a multi-objective choice whose efficient set changes in a predictable way with task length. For sustainability research, it shows how a choice of functional unit changes the environmental comparison of AI systems.

The rest of the paper follows the selection question step by step. Section~\ref{sec:related} reviews what is known about it and identifies four gaps. Section~\ref{sec:motivation} uses a simple example to show why per-call comparisons can mislead. Section~\ref{sec:methodology} formalises our answer as a general decision model and a two-stage framework, with an illustrative example. Section~\ref{sec:evaluation} asks how the framework can be tested, and Section~\ref{sec:experiments} describes the experiments and reports their results. Section~\ref{sec:discussion} discusses what the results mean, the limits of the approach and the main objections to it. Section~\ref{sec:conclusion} returns to the selection question and sets out directions for future work.

\section{Literature review\label{sec:related}}

To answer the selection question, we need three things: a way to measure how well an agent completes tasks, a way to choose between models, and a way to count the resources that each choice uses. This section reviews the literature on each of them in turn, then the decision-analysis tools we build on, and ends with the gaps that motivate our work.

\subsection{How AI agents are evaluated}

Recent surveys describe LLM agents in terms of reasoning, tool use, coordination between several agents and memory \citep{lei2025agents}, and catalogue the protocols used to coordinate agents in practice \citep{zhu2026orchestration}. Both show that an agent depends on abilities that a single-turn benchmark does not test: planning over many steps, calling tools correctly and staying consistent over a long sequence of actions. Evaluation research has followed this shift. \citet{yehudai2025survey} survey agent evaluation and find that cost-efficiency, safety and robustness are rarely measured under realistic conditions. \citet{kapoor2025agentsmatter} argue that agents should be compared on accuracy and cost together, and show that simple agents can match complex ones at a fraction of the cost. Benchmarks such as $\tau$-bench and its successor \citep{yao2024taubench,barres2025tau2bench} test agents in realistic customer-service tasks, and measure how often an agent succeeds on every one of several repeated trials. The Berkeley Function Calling Leaderboard measures how reliably models call tools, from single calls to multi-turn agentic tasks \citep{patil2025bfcl}. Other work shows that outcome-only evaluation can hide failures such as skipped steps or misused tools \citep{gritta2026process}, and that agents can be misled by unreliable tools \citep{omar2026poisoned}.

A second line of work studies how success depends on the length of a task. \citet{kwa2025longtasks} measure the length of tasks, in human working time, that models can complete with 50\% success, and find that this length has doubled about every seven months. \citet{sinha2026illusion} show that small gains in the accuracy of single steps lead to large gains in the length of tasks a model can complete, and that errors become more likely when earlier mistakes remain in the context. \citet{mittal2026agentsrot} find that task success falls with the number of steps in a way that is well described by a single per-step reliability, which rises with model size but stays below one even for the strongest models. These studies explain why agents fail on long tasks. They do not ask what this means for the choice between models when energy, cost and time are also at stake.

\subsection{How models are chosen today}

A large literature treats model choice as a trade-off between quality and cost, mostly for single requests. FrugalGPT \citep{chen2024frugalgpt} and AutoMix \citep{aggarwal2023automix} send a request to a cheap model first and pass it to a stronger model only when needed. Other methods learn to predict which model will answer a request well enough: from past performance \citep{sakota2024flyswat}, from human preference data \citep{ong2025routellm}, from the predicted difficulty of the request \citep{ding2024hybridllm}, or from item response theory \citep{song2025irtrouter}. \citet{shah2025selectthenroute} first narrow a large set of models and then route within it, and RouterBench \citep{hu2024routerbench} offers a common benchmark for such methods. More recent work routes the steps of agentic workflows. \citet{fu2026inflation} show that the true cost of a workflow can be several times the cost of a single call, because of retries and longer reasoning, and \citet{li2026progrouter} choose a model at each step based on the progress of the task. From an operational research view, \citet{yang2026reliableworkflows} write the reliability of a sequential workflow as the product of the reliabilities of its steps, give matching rules for parallel steps and feedback loops, and allocate a token budget between steps so that the workflow is as reliable as possible within limits on cost and latency. Their question is how to configure a given set of models; the product rule they use is also the starting point of our model. \citet{erol2026costofpass} treat a model as a producer whose output is uncertain, and measure the expected cost of one correct answer: the cost of an attempt divided by its chance of success. The cheapest model by this measure depends on the kind of problem. Their unit is a single problem and their resource is money. We apply the same ratio to tasks of many steps, to energy and time as well as cost, and ask how the answer changes with task length. Across this literature, energy and carbon usually enter through cost, and the length of the task is not yet a variable in the choice of model. These are the two points we take up.

A small number of studies treat model choice as a multi-criteria decision. The closest to ours is \citet{nowak2026framework}, which frames the routing of business prompts between a cheaper and a stronger model as a managerial decision. Criterion weights are elicited from experts with the Analytic Hierarchy Process (AHP) \citep{saaty1977ahp,saaty2008ahp}, and each prompt is scored on seven criteria before Simple Additive Weighting (SAW) \citep{hwang1981madm,kaliszewski2016saw} chooses the model. On 500 business prompts the framework reaches 94.4\% response sufficiency, against 94.6\% for always using the stronger model, at 37.4\% lower cost. HugSelect \citep{joonbakhsh2026hugselect} ranks foundation models from public repositories with a weighted additive model. Both work with single requests, both apply one weighted score to all alternatives, and neither considers energy or the length of a task.

\subsection{How the resources of AI inference are counted}

The energy and carbon cost of AI has been studied since \citet{strubell2019energy} drew attention to the cost of training. For inference, \citet{luccioni2024powerhungry} show that general-purpose generative models use much more energy than task-specific ones. The ML.ENERGY benchmark \citep{chung2025mlenergy} measures the energy of inference for open models on known hardware, \citet{jegham2025hungry} estimate the energy, water and carbon of commercial models, and \citet{elsworth2025google} report measured figures for a large commercial service. \citet{oviedo2026energy} estimate the energy of inference at scale and note that agentic workflows, with many calls per user request, can have a large effect on the total. \citet{kim2026dynamicreasoning} measure the resources, latency and energy of several agent designs, and find that extra computation raises accuracy with diminishing returns and at a high cost in infrastructure. Most of this work counts energy per token or per request. Two recent studies move towards the task. \citet{panigrahy2026epg} propose \emph{energy per successful goal}: the energy of all attempts at a set of goals, failed ones included, divided by the number of goals achieved. With careful measurement on the processor, they find that agentic workflows use about four times as much energy per successful goal as a single call on the same tasks, because of planning and coordination steps as well as retries, and that the order reverses when a tool call replaces long text generation. They also point to the choice of model as a use of the unit, since a larger model may need fewer retries. \citet{liu2026costutility} propose to compare the resources used by an agent with its contribution to the task. On the side of choice, \citet{sheshanarayana2026gar} route single requests to minimise carbon subject to accuracy and latency targets, and \citet{belcak2025slm} argue that small models are sufficient, and more economical, for most agentic tasks. \citet{wang2026rethinkingscale} compare small open models used alone, as a single agent with tools and as a team of agents, and find that a single agent gives the best balance of performance and cost, while a team adds overhead for little gain. Our work follows this line by building a per-task measure of resources into a model of the choice between models.

\subsection{Decision-analysis tools}

Our framework builds on standard tools of multi-criteria decision analysis \citep{Keeney1993,belton2002mcda,greco2016multiple}. Multi-attribute decision making chooses among a finite set of alternatives that are known in advance, while multi-objective decision making searches a set of solutions defined by constraints, often infinite, for those that are efficient, that is, no worse than another solution on every objective and strictly better on none \citep{zeleny1982mcdm,miettinen1999nonlinear}. Choosing a language model is a multi-attribute problem: the candidates are a finite list of available models. It has several objectives, however, and we borrow from multi-objective decision making the idea of an efficient set, which here is found by comparing the candidates in pairs. Removing dominated alternatives before weighting the rest is itself standard practice in multi-attribute analysis \citep{belton2002mcda}. Efficient sets can be found by non-dominated sorting \citep{deb2001moea}, and ranked with methods such as SAW and TOPSIS \citep{hwang1981madm,ciardiello2023topsissaw}. Weights can be elicited with AHP or, when the decision maker can give only partial information, with the flexible and interactive trade-off method \citep{dealmeida2016fitradeoff}. Three further ideas are central to our model. \emph{Satisficing} means choosing an alternative that is good enough rather than the best \citep{simon1955behavioral}. The \emph{$\varepsilon$-constraint} method turns some objectives into constraints and shows that the solutions found in this way lie on the efficient set \citep{haimes1971bicriterion}. \emph{Reference-point} methods let the decision maker set reservation and aspiration levels for each objective \citep{wierzbicki1980reference}. Finally, the success of a task whose steps must all succeed follows the reliability of a series system, a classical result of reliability theory \citep{barlow1975reliability}.

\subsection{Gaps\label{sec:gaps}}

Returning to the selection question, the review shows that the parts of an answer exist but have not been put together. Table~\ref{tab:gaps} summarises four gaps and how this paper addresses them.

\begin{table}[htbp]
\centering
\caption{Gaps in the literature on the selection question, and how this paper addresses them.}
\label{tab:gaps}
\begin{tabularx}{\textwidth}{p{3.2cm}XX}
\toprule
Gap & What is missing & How we address it \\
\midrule
G1. Unit of comparison & Models are chosen on per-call measures. Energy per successful task has been measured and proposed as a guide to design, but not yet built into a model of choice. & A decision model built on resources per successful task (Section~\ref{sec:general}). \\
\addlinespace
G2. When is smaller better? & Task success is known to fall with task length, but there is no rule for when this reverses the ranking of models on energy, cost or time. & Crossover lengths and three regimes (Section~\ref{sec:properties}). \\
\addlinespace
G3. Role of preferences & Removing dominated alternatives before weighting is standard in multi-criteria analysis. In model choice, however, the efficient set itself changes with task length, and it is not known for which tasks preferences are needed at all. & Intervals of task length over which the efficient set stays the same, and the bands in which a weighted choice is needed (Sections~\ref{sec:properties} and \ref{sec:stage2}). \\
\addlinespace
G4. Kinds of task & Methods do not distinguish binary from graded outcomes, or actions that can be retried from actions that cannot. & Four special cases of one general model (Section~\ref{sec:variants}). \\
\bottomrule
\end{tabularx}
\end{table}

\section{Motivation: why per-call comparisons can mislead\label{sec:motivation}}

Before we formalise our answer to the selection question, we use a simple example to show why the unit of comparison matters.

Suppose model $a$ completes each step of a task correctly 95\% of the time and model $b$ 99\% of the time. On a single call the gap looks small. Now consider a customer-service task that needs 20 steps and succeeds only if every step succeeds. Model $a$ completes the task with probability $0.95^{20} \approx 0.36$ and model $b$ with probability $0.99^{20} \approx 0.82$. A failed attempt still uses energy, money and time. If failed tasks are repeated until they succeed, model $a$ needs on average about 2.8 attempts for each success and model $b$ about 1.2. Suppose model $a$ uses one tenth of the energy of model $b$ per call. Per completed task, model $a$ still uses less energy, but its advantage has fallen from ten times to about four times. For a task of 60 steps, model $b$ uses less energy per completed task than model $a$. The model that looked greener per call is no longer greener.

Life-cycle assessment faces the same problem when it compares products. It solves it with the \emph{functional unit}, the quantity of service to which all impacts are related \citep{iso14044}. Delivery services, for example, are compared per delivered parcel, not per vehicle. We argue that for agentic AI the functional unit should be one successfully completed task. \citet{panigrahy2026epg} make the same case for agentic AI, drawing on the example of power usage effectiveness in data centres. They show with a simple example that a system with a higher energy per call can use less energy per goal if it needs fewer attempts. \citet{kapoor2025agentsmatter} argue in a similar spirit that the cost of agents should be reported together with their accuracy. We take this unit as the basis for choosing between models, and ask how the answer changes with the length of the task.

The example raises four questions that a framework for the selection question must answer:
\begin{enumerate}[label=(Q\arabic*)]
  \item Which models can be ruled out for every task, whatever its length?
  \item For a task of a given length, which models are efficient, that is, not worse than another model on energy, cost and time together?
  \item When does the choice between efficient models depend on the preferences of the decision maker?
  \item How does the answer change when actions cannot be undone, or when the quality of the outcome is graded rather than binary?
\end{enumerate}
Section~\ref{sec:methodology} answers these questions in turn.

\section{A framework for the selection question\label{sec:methodology}}

This section formalises our answer to the selection question. We first give an overview of the framework and of what it measures. We then present a general decision model, state its main properties, describe its special cases and relax its assumptions. We close with the choice within a trade-off, an illustrative example and a precise definition of right-sizing.

\subsection{Overview}

Let $\mathcal{A} = \{a_1, \ldots, a_m\}$ denote the set of candidate models. More precisely, each candidate is a model offered by a provider, since the same model served by two providers can differ in price, speed and terms of use. The set changes over time, as new models appear and old ones are withdrawn. A \emph{task profile} describes the task for which a model is needed: how many steps it takes, how good the outcome must be, whether a failed attempt can be repeated, how important it is that the task is completed, and how much the deployment cares about energy, cost and time.

The framework maps a set of candidates and a task profile to a recommendation in four steps:
\begin{enumerate}[label=(\roman*)]
  \item a \emph{feasibility screen} (Stage~0) removes candidates that break a hard constraint, such as a data-protection rule;
  \item a \emph{per-step screen} (Stage~1) removes candidates that are worse than another candidate on every per-step measure; by Observation~\ref{obs:prune} below, these candidates can never be efficient, for any task;
  \item a \emph{task-level screen} finds the candidates that are efficient for the task, measured per successfully completed task, and that meet the required chance of completing it;
  \item a \emph{preference-based choice} (Stage~2) picks one of the efficient candidates, but only when more than one remains.
\end{enumerate}
The output is the recommended model together with the efficient set and, where Stage~2 was needed, the weights used. This makes each recommendation open to audit. Figure~\ref{fig:pipeline} shows the four steps and the information each one uses.

\begin{figure}[htbp]
\centering
\begin{tikzpicture}[
  node distance=5mm and 7mm,
  stage/.style={draw, rounded corners=2pt, align=center, text width=5.6cm, minimum height=11mm, font=\small, fill=gray!12},
  pref/.style={draw, rounded corners=2pt, align=center, text width=5.6cm, minimum height=11mm, font=\small, fill=white, dashed},
  input/.style={draw, align=center, text width=3.2cm, minimum height=11mm, font=\small},
  decision/.style={draw, diamond, aspect=2.4, align=center, font=\small, inner sep=0.5pt},
  result/.style={draw, double, rounded corners=2pt, align=center, text width=5.6cm, minimum height=10mm, font=\small},
  arr/.style={-{Stealth[length=2.2mm]}, thick},
  feed/.style={-{Stealth[length=1.8mm]}, densely dotted}
]
\node[stage] (s0) {\textbf{Stage 0: feasibility screen}\\ remove models that break a hard constraint (data rules, context, safety, ceilings)};
\node[stage, below=of s0] (s1) {\textbf{Stage 1: per-step screen}\\ remove models worse on every per-step measure; valid for every task length};
\node[stage, below=of s1] (s2) {\textbf{Task-level screen}\\ resources per successful task, $R_{a,k}(L) = r_{a,k}L/S_a(L)$; keep feasible models with $P_a(L,K) \ge \beta$ that are efficient};
\node[decision, below=of s2] (d) {More than one\\ efficient model?};
\node[pref, below=of d] (st2) {\textbf{Stage 2: preference-based choice}\\ weights on energy, cost and time; SAW with a TOPSIS check};
\node[result, below=of st2] (o) {\textbf{Recommendation}\\ with the efficient set and, if used, the weights};

\node[input, left=of s1] (m) {\textbf{Model measurements}\\ per task family:\\ $\rho_a$, $r_{a,e}$, $r_{a,c}$, $r_{a,t}$};
\node[input, right=of s1] (t) {\textbf{Task profile}\\ $L$, $\underline{q}$, $K$, $\beta$,\\ constraints $\Gamma$, weights};

\draw[arr] (s0) -- (s1);
\draw[arr] (s1) -- (s2);
\draw[arr] (s2) -- (d);
\draw[arr] (d) -- node[right, font=\small] {yes} (st2);
\draw[arr] (st2) -- (o);
\draw[arr] (d.west) -- ++(-1.2,0) node[above, font=\small, pos=0.5] {no} |- (o.west);

\draw[feed] (m) -- (s1);
\draw[feed] (m.south) |- (s2.west);
\draw[feed] (t.north) |- (s0.east);
\draw[feed] (t.south) |- (s2.east);
\draw[feed] (t.south) ++(0.6,0) |- (st2.east);
\end{tikzpicture}
\caption{The framework. Shaded steps use no preference information; Stage~2 (dashed) uses the weights of the task profile and runs only when more than one model is efficient. Dotted arrows show which inputs each step uses.}
\label{fig:pipeline}
\end{figure}
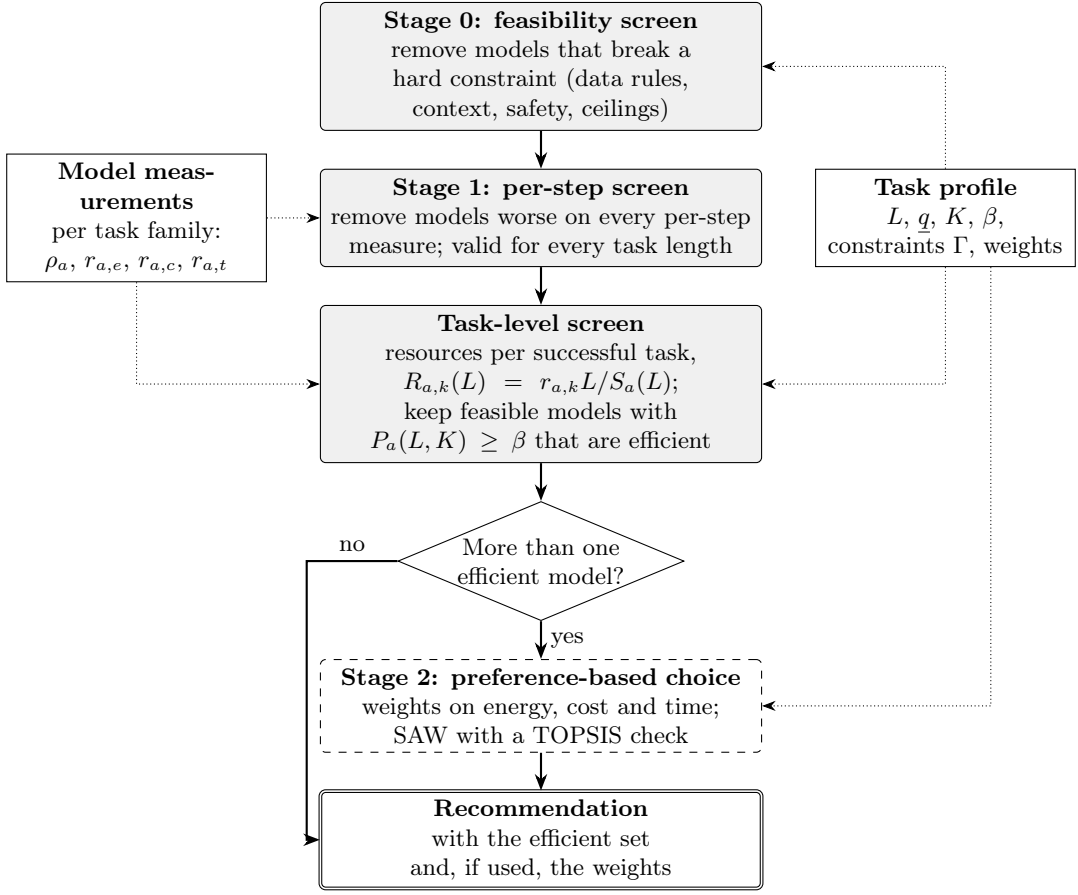

\subsection{What we measure\label{sec:criteria}}

The framework separates properties of a model from properties of a task. Table~\ref{tab:objectives} lists the model quantities. Per-step reliability is measured for a \emph{task family}, such as customer-service tasks with tool calls, because a model can be reliable in one kind of task and not in another. When no measurement for the task family exists, a general tool-calling score can serve as a first estimate. Context capacity, safety and data governance enter as constraints, not as objectives. The per-step quantities are not themselves the objectives. The objectives are energy, cost and time per successfully completed task (Section~\ref{sec:general}), and reliability enters them, and the chance of completing the task, through the success probability.

\begin{table}[htbp]
\centering
\caption{Model quantities used by the framework. Per-step quantities are measured for a task family.}
\label{tab:objectives}
\begin{tabularx}{\textwidth}{llX}
\toprule
Quantity & Role & Measurement \\
\midrule
$\rho_a$: per-step reliability & parameter & share of steps completed correctly in benchmark runs of the task family \\
$r_{a,e}$: energy per step & parameter & measured on known hardware for open-weight models; estimated with bounds for closed models \\
$r_{a,c}$: cost per step & parameter & tokens per step multiplied by the provider's price \\
$r_{a,t}$: time per step & parameter & measured response time per step \\
Context capacity & constraint & provider specification \\
Safety & constraint & published safety evaluations and robustness tests \\
Governance & constraint & deployment class, data residency, licence \\
\bottomrule
\end{tabularx}
\end{table}

Energy is the one quantity that cannot always be measured directly, since providers of closed models do not report it. Each value therefore carries a flag that says whether it was measured or estimated, and estimated values carry a range. Every value also carries a date and a source, because the market changes and the record of where numbers came from is part of the audit trail.

The task profile is described by the parameters in Table~\ref{tab:params} below. They can be set in two ways. In the \emph{deliberative} way, a team sets them once for an agent or a type of task, through a short questionnaire. In the \emph{automated} way, a small model reads each task description and estimates them at run time, as the router of \citet{nowak2026framework} does for its criteria. The rest of the framework is the same in both cases.

\subsection{Stage 0: feasibility screen}

Let $\Gamma$ denote the hard constraints that follow from the task profile: the context the task needs, the deployment classes allowed by data-protection rules, a minimum safety level, any list of allowed or excluded providers, and any ceiling on cost, time or carbon per task. The feasible set is
\begin{equation}
  \mathcal{A}_\Gamma \;=\; \{\, a \in \mathcal{A} \;:\; a \text{ meets every constraint in } \Gamma \,\}.
\end{equation}
These constraints come first because no amount of reliability or efficiency can make up for a model that may not legally process the data.

\subsection{A general model of resources per successful task\label{sec:general}}

We now give the general model that answers the selection question for a single task. Consider a candidate model $a$ and a task that needs $L$ steps, where a step is one call to the model. We write $k \in \{e, c, t\}$ for the three resources, energy, cost and time, so that $r_{a,k}$ is the amount of resource $k$ that model $a$ uses per step.

An attempt at a task is \emph{successful} if it reaches the goal of the task: the final state is correct, for example the booking is changed as the customer asked. This is the unit on which all models are compared. The outcome of one attempt may be binary (the goal is reached or not) or graded, for example a quality score $q$ between 0 and 1. The decision maker sets a \emph{quality floor} $\underline{q}$, and an attempt counts as a success if its quality reaches the floor. For binary outcomes $\underline{q} = 1$. We write $S_a(L)$ for the probability that one attempt by model $a$ succeeds. Reaching the goal is not the same as getting every step right, since an agent can sometimes notice an error and correct it in a later step. In the \emph{basic case} we take the stricter view: steps succeed or fail independently and any failed step spoils the attempt, as in a series system \citep{barlow1975reliability} and as in the sequential workflows of \citet{yang2026reliableworkflows}, so that
\begin{equation}
  S_a(L) = \rho_a^{\,L}.
  \label{eq:basic}
\end{equation}
Section~\ref{sec:relax} relaxes this assumption and allows errors to be corrected.

A failed attempt may be repeated. We allow at most $K$ attempts per task. $K = 1$ describes an \emph{irreversible} task, such as a payment that cannot be undone. $K = \infty$ describes a task that can be retried freely, such as a data extraction whose result is checked automatically. We assume that each attempt runs all $L$ steps before its outcome is known. This is the cautious case, because it charges every failed attempt in full.

\begin{theorem}[Resource per successful task]\label{obs:lemma}
Let $S_a(L) > 0$ and $K \ge 1$. The probability that model $a$ completes the task within $K$ attempts is
\begin{equation}
  P_a(L, K) = 1 - \bigl(1 - S_a(L)\bigr)^{K},
  \label{eq:complete}
\end{equation}
and the expected amount of resource $k$ used per successfully completed task, counting the resources of failed attempts as well as of the successful one, is
\begin{equation}
  R_{a,k}(L) = \frac{r_{a,k}\, L}{S_a(L)},
  \label{eq:persuccess}
\end{equation}
whatever the value of $K$.
\end{theorem}

\noindent\textit{Proof.} Attempts succeed independently with probability $S = S_a(L)$, which gives \eqref{eq:complete}. The expected number of attempts, with at most $K$ allowed, is $\sum_{i=0}^{K-1} (1-S)^i = \bigl(1-(1-S)^K\bigr)/S$. Each attempt uses $r_{a,k} L$ of resource $k$. Dividing the expected use by the probability of success within $K$ attempts gives \eqref{eq:persuccess}. $\square$

Observation~\ref{obs:lemma} says that the resource per successful task does not depend on how many retries are allowed. The retry limit changes how often the task is completed at all, but not how efficiently each success is produced. The same ratio underlies the cost-of-pass of \citet{erol2026costofpass} for single problems, and \citet{panigrahy2026epg} reach the same result for energy per successful goal when retries are independent, and note that it no longer holds when later attempts differ from earlier ones, a point we return to in Section~\ref{sec:discussion}. This lets us separate two questions: whether a model is reliable enough for the task, which is a constraint, and how much it uses to deliver each success, which gives the objectives.

The general decision problem follows. For a task with length $L$, quality floor $\underline{q}$, retry limit $K$ and a required probability of completion $\beta$, choose a model from the feasible set $\mathcal{A}_\Gamma$ that
\begin{equation}
\begin{aligned}
  \text{minimises}\quad & \bigl(R_{a,e}(L),\; R_{a,c}(L),\; R_{a,t}(L)\bigr)
     \quad \bigl[\text{and, optionally, maximises } \bar{q}_a(L)\bigr] \\
  \text{subject to}\quad & P_a(L, K) \ge \beta ,
\end{aligned}
\label{eq:general}
\end{equation}
where $\bar{q}_a(L)$ is the expected quality of a successful output. The problem has several objectives, and we do not combine them at this point. We first find the models that are efficient for \eqref{eq:general}, and leave the choice among them to Stage~2. The constraint is a \emph{chance constraint}: it requires the task to be completed with probability at least $\beta$. By \eqref{eq:complete}, it is the same as a lower bound on the success of a single attempt,
\begin{equation}
  S_a(L) \;\ge\; 1 - (1-\beta)^{1/K}.
  \label{eq:floor}
\end{equation}
For example, with $\beta = 0.9$ and $K = 3$, each attempt must succeed with probability at least $0.54$. Table~\ref{tab:params} lists the parameters of the model, who sets them, and how they relate to the criteria of task profiles used in earlier work \citep{nowak2026framework}.

\begin{table}[htbp]
\centering
\caption{Parameters of the general model \eqref{eq:general}.}
\label{tab:params}
\begin{tabularx}{\textwidth}{llX}
\toprule
Symbol & Set by & Meaning and example \\
\midrule
$\rho_a$ & measurement & Per-step reliability of model $a$ on the task family. \\
$r_{a,e}, r_{a,c}, r_{a,t}$ & measurement & Energy, cost and time per step of model $a$. \\
$L$ & task profile & Number of steps in one attempt; follows from the expected length of the task. \\
$\underline{q}$ & decision maker & Quality floor; follows from the required accuracy. $\underline{q} = 1$ for tasks that are simply done or not done. \\
$K$ & decision maker & Maximum number of attempts; follows from whether actions can be undone. $K = 1$ for irreversible actions, $K = \infty$ for free retries. \\
$\beta$ & decision maker & Required probability that the task is completed; follows from the risk of a failed action. \\
\bottomrule
\end{tabularx}
\end{table}

\subsection{Three properties of the general model\label{sec:properties}}

We now state three properties of the model in the basic case \eqref{eq:basic}. They answer questions (Q1) to (Q3) of Section~\ref{sec:motivation}. Each has a short proof and a plain reading. Throughout, model $a$ \emph{dominates} model $b$ if $a$ is no worse than $b$ on every objective and strictly better on at least one, and a model is \emph{efficient} if no other feasible model dominates it.

\begin{theorem}[Pruning that holds for every task length]\label{obs:prune}
If $\rho_b \ge \rho_a$ and $r_{b,k} \le r_{a,k}$ for every resource $k$, with at least one of these inequalities strict, then $b$ dominates $a$ for every task length $L \ge 1$, and $b$ meets the constraint \eqref{eq:floor} whenever $a$ does.
\end{theorem}

\noindent\textit{Proof.} By \eqref{eq:basic} and \eqref{eq:persuccess}, $R_{a,k}(L) = r_{a,k} L / \rho_a^{L}$, which increases with $r_{a,k}$ and decreases with $\rho_a$. Hence $R_{b,k}(L) \le R_{a,k}(L)$ for every $k$. The inequality is strict for every $k$ if $\rho_b > \rho_a$, and strict for resource $k$ if $r_{b,k} < r_{a,k}$. The constraint depends on the model only through $\rho^L$, which is at least as large for $b$. $\square$

In plain terms, a model that is worse than another on every per-step measure can be removed once, for all tasks. This answers (Q1), and it is the Stage~1 screen of the framework. The screen has two useful properties. First, it depends only on the order of the candidates on each per-step quantity, so no choice of units, scaling or weights can change it. Second, it uses no information about the task beyond the task family, so one screen serves every task in that family. The costly part of the work, measuring the candidates, is shared across all tasks and users.

\begin{theorem}[One crossover for each resource]\label{obs:cross}
Let $\rho_a < \rho_b$ and $r_{a,k} < r_{b,k}$ for a resource $k$. Then
\begin{equation}
  \ln \frac{R_{a,k}(L)}{R_{b,k}(L)} \;=\; \ln \frac{r_{a,k}}{r_{b,k}} \;+\; L \,\ln \frac{\rho_b}{\rho_a},
  \label{eq:lines}
\end{equation}
and model $a$ uses less of resource $k$ per successful task than model $b$ if and only if
\begin{equation}
  L \;<\; L^{*}_{k} \;=\; \frac{\ln (r_{b,k}/r_{a,k})}{\ln (\rho_b/\rho_a)} .
  \label{eq:crossover}
\end{equation}
\end{theorem}

\noindent\textit{Proof.} Equation \eqref{eq:lines} follows directly from \eqref{eq:persuccess} and \eqref{eq:basic}. Its right-hand side is negative exactly when $L < L^{*}_{k}$. $\square$

In plain terms, the less reliable model starts with an advantage on each resource, and this advantage shrinks as tasks get longer. Equation~\eqref{eq:lines} is a straight line in $L$, and its slope $\ln(\rho_b/\rho_a)$ is \emph{the same for energy, cost and time}. Only the starting advantage differs between the three resources. Each resource therefore has its own crossover length $L^{*}_{k}$, and the resource on which the less reliable model has the largest starting advantage is the last to cross.

\begin{theorem}[Three regimes]\label{obs:regimes}
Let $\rho_a < \rho_b$ and $r_{a,k} < r_{b,k}$ for all three resources, and let $L_{\min} = \min_k L^{*}_{k}$ and $L_{\max} = \max_k L^{*}_{k}$. Then
\begin{enumerate}[label=(\roman*)]
  \item for $L < L_{\min}$, model $a$ dominates model $b$;
  \item for $L > L_{\max}$, model $b$ dominates model $a$;
  \item for $L_{\min} < L < L_{\max}$, neither model dominates the other.
\end{enumerate}
The middle band is empty when the ratios $r_{b,k}/r_{a,k}$ are equal for all three resources.
\end{theorem}

\noindent\textit{Proof.} Parts (i) to (iii) follow from Observation~\ref{obs:cross} applied to each resource. If the three ratios are equal, the three crossover lengths in \eqref{eq:crossover} are equal and $L_{\min} = L_{\max}$. $\square$

In plain terms, for short tasks the less reliable model is better on every resource, and for long tasks the more reliable model is. In both cases the choice needs no preference information. Only in the band between them does the choice depend on what the decision maker values, for example low carbon over speed. This answers (Q2) and (Q3). The width of the band depends on how differently the three resources compare. If energy, cost and time all differ between two models by the same factor, for example because price simply follows energy use, the band disappears and preferences are never needed for that pair. When tasks can be retried freely, the efficient set can change only at crossover lengths, of which there are at most $3m(m-1)/2$ for $m$ candidate models. When the chance constraint \eqref{eq:floor} applies, the set can also change where a model stops meeting it. In the basic case, model $a$ meets the constraint if and only if
\begin{equation}
  L \;\le\; \bar{L}_a \;=\; \frac{\ln\bigl(1 - (1-\beta)^{1/K}\bigr)}{\ln \rho_a},
  \label{eq:feasible}
\end{equation}
which adds at most $m$ \emph{feasibility thresholds}. A model that reaches its threshold leaves the feasible set, and a model that it dominated may then become efficient. Together, the crossover lengths and the feasibility thresholds divide task length into at most $3m(m-1)/2 + m + 1$ intervals. Within each interval the efficient set stays the same, so it can be computed once per interval and stored. This answers a question that a weighted choice made separately for each task cannot: for which task lengths the choice does not depend on preferences at all.

\subsection{Special cases and variants\label{sec:variants}}

We now turn to question (Q4). The general model has four main special cases. They follow from two questions about the task: is the outcome binary or graded, and can a failed attempt be repeated? Table~\ref{tab:variants} summarises them.

\begin{table}[htbp]
\centering
\caption{Four special cases of the general model \eqref{eq:general}.}
\label{tab:variants}
\begin{tabularx}{\textwidth}{lXX}
\toprule
 & Failures can be retried ($K = \infty$) & Irreversible or limited retries ($K$ finite) \\
\midrule
Binary outcome ($\underline{q} = 1$) & \textbf{V1.} Minimise energy, cost and time per successful task. No constraint. \newline \emph{Example:} data extraction with an automatic check. & \textbf{V2.} Same objectives, subject to the chance constraint \eqref{eq:floor}. \newline \emph{Example:} changing an airline booking. \\
\addlinespace
Graded outcome ($\underline{q} < 1$) & \textbf{V3.} Success means reaching the quality floor. Quality above the floor has no extra value. \newline \emph{Example:} a summary that must contain a set of required facts. & \textbf{V4.} Quality floor as a constraint and expected quality $\bar{q}_a$ as a fourth objective. \newline \emph{Example:} drafting a report, where better quality has real value. \\
\bottomrule
\end{tabularx}
\end{table}

\paragraph{V1: binary outcome, free retries} This is the pure efficiency question. Observations~\ref{obs:prune} to~\ref{obs:regimes} apply as stated.

\paragraph{V2: binary outcome, irreversible actions} By Observation~\ref{obs:lemma}, the objectives are the same as in V1, and only the feasible set changes. The constraint \eqref{eq:floor} can be strict. By \eqref{eq:feasible}, with $K = 1$ and $\beta = 0.9$, a model with $\rho = 0.95$ meets it only for tasks of up to 2 steps, and a model with $\rho = 0.99$ only for tasks of up to 10 steps. Beyond these lengths neither model is feasible. This result is useful on its own, because it tells the decision maker that the task should be redesigned, for example by adding a human check before each irreversible step, rather than given to a different model.

\paragraph{V3: graded outcome with a quality floor} Success now means that the quality of the output reaches the floor, so $S_a(L)$ depends on $\underline{q}$. Quality above the floor has no extra value. This is the \emph{satisficing} rule of \citet{simon1955behavioral}: the decision maker asks for the least resource-intensive model that is good enough.

\paragraph{V4: graded outcome, floor and maximisation} Here the decision maker also values quality above the floor, so expected quality becomes a fourth objective. The floor acts as a reservation level, below which no output is acceptable, and above it quality is traded against resources. This is the logic of reference-point methods \citep{wierzbicki1980reference}. V3 and V4 are not separate frameworks. Following the $\varepsilon$-constraint idea \citep{haimes1971bicriterion}, any objective can be turned into a constraint, and the choices made under V3 and V4 both lie on the efficient set of the same general problem. The difference between the V3 and V4 recommendations measures the \emph{price of quality}: the extra energy, cost and time that the decision maker pays for quality above the floor. Observation~\ref{obs:cross} still holds for the three resources. Quality must, however, be part of every comparison. In the per-step screen of Observation~\ref{obs:prune}, model $b$ removes model $a$ only if it is also no worse on expected quality, and the same holds in the task-level screen and in Stage~2, where quality becomes a fourth criterion. The quality objective usually keeps more models efficient, so the bands in which preferences matter become wider.

\paragraph{The effect of the quality floor} Raising the floor lowers the success probability of every model. If it lowers the success of the smaller model more than that of the larger model, which we expect and test in Section~\ref{sec:evaluation}, the crossover lengths move to shorter tasks. The efficient choice then depends on two features of the task, its length and its quality floor. It can be shown as a map with three regions: the smaller model dominates, both are efficient, and the larger model dominates.

\subsection{Relaxing the assumptions\label{sec:relax}}

The basic case makes three simplifying assumptions. We now relax each one and state what remains true.

\paragraph{Reliability that changes along the task} Per-step reliability need not be constant. Errors can become more likely as a task goes on, for example when earlier mistakes remain in the context \citep{sinha2026illusion}. Survival analysis of multi-turn conversations models the turn at which a model first fails, and shows how the risk of that failure changes from turn to turn \citep{li2026timetoinconsistency}. In this case we replace $\rho_a^{L}$ by the measured success curve $S_a(L)$. The same step covers agents that correct some of their errors, because $S_a(L)$ then measures how often the goal is reached, not how often every step is right. Observation~\ref{obs:lemma} does not depend on the form of $S_a(L)$ and holds unchanged. Observation~\ref{obs:prune} holds if $S_b(L) \ge S_a(L)$ for every $L$. The single crossover of Observation~\ref{obs:cross} holds if the ratio $S_a(L)/S_b(L)$ falls as $L$ grows. While resource use per step stays constant, the lines in \eqref{eq:lines} are then no longer straight, but they still differ only in their starting values, so the three regimes of Observation~\ref{obs:regimes} remain. The condition on $S_a(L)/S_b(L)$ can be checked directly from benchmark data, and we test it in Section~\ref{sec:evaluation}.

\paragraph{Resource use that changes along the task} The amount of a resource used per step need not be constant either. In our pilot (Section~\ref{sec:experiments}), the energy of a step grows as the conversation gets longer, by up to about four times over 150 steps for the smallest models and by a quarter or less for the largest. Let $r_{a,k}(j)$ be the amount of resource $k$ that model $a$ uses at step $j$, and let $C_{a,k}(L) = \sum_{j=1}^{L} r_{a,k}(j)$ be the amount used by one attempt of $L$ steps. Observation~\ref{obs:lemma} holds with $r_{a,k}L$ replaced by $C_{a,k}(L)$, so that $R_{a,k}(L) = C_{a,k}(L)/S_a(L)$. The basic model is the special case in which $r_{a,k}(j) = r_{a,k}$ at every step, and reliability and resource use may now both change along the task. 

The pruning rule of Observation~\ref{obs:prune} extends to
variable per-step resource use if
$C_{b,k}(L) \leq C_{a,k}(L)$ for every resource $k$ and
$S_b(L) \geq S_a(L)$ for all $L$, with strict dominance
in at least one objective at each task length.
Under these conditions, model $a$ can be permanently
discarded. Comparing average per-step consumption
over a fixed horizon alone does not guarantee this property.

\begin{theorem}[One crossover when steps vary]\label{obs:growth}
Let $g_k(L) = C_{a,k}(L)/C_{b,k}(L)$ be the ratio of the resources used by one attempt of models $a$ and $b$, and $s(L) = S_a(L)/S_b(L)$ the ratio of their success probabilities. If, for every $L \ge 1$,
\begin{equation}
  \frac{g_k(L+1)}{g_k(L)} \;\ge\; \frac{s(L+1)}{s(L)},
  \label{eq:growth}
\end{equation}
then there is a length $L^{*}_{k}$, possibly infinite, such that model $a$ uses less of resource $k$ per successful task than model $b$ if and only if $L < L^{*}_{k}$.
\end{theorem}

\noindent\textit{Proof.} By the extended Observation~\ref{obs:lemma}, $\ln\bigl(R_{a,k}(L)/R_{b,k}(L)\bigr) = \ln g_k(L) - \ln s(L)$. Condition \eqref{eq:growth} says that this quantity does not fall from one length to the next. A sequence that never falls is negative on an initial run of lengths, possibly empty or unending, and not negative after it. $\square$

In plain terms, in the basic case the right-hand side of \eqref{eq:growth} is $\rho_a/\rho_b$, which is below one when model $a$ is the less reliable. The condition then says that the resource ratio of the two models may fall as the task gets longer, but by no more than the factor $\rho_a/\rho_b$ per step. Three cases follow.
\begin{enumerate}[label=(\roman*)]
  \item \emph{Proportional growth.} If both models follow the same profile, $r_{a,k}(j) = r_{a,k}\, h_k(j)$ and $r_{b,k}(j) = r_{b,k}\, h_k(j)$ for some common $h_k$, then $g_k$ is constant, the condition holds, and in the basic case the crossover length is exactly \eqref{eq:crossover}. The formula therefore holds whenever two models grow at the same rate, whatever that rate is.
  \item \emph{The less reliable model grows faster.} Then $g_k$ rises, the condition holds, and in the basic case $\ln\bigl(R_{a,k}(L)/R_{b,k}(L)\bigr) \ge \ln g_k(1) + L \ln(\rho_b/\rho_a)$. The crossover therefore comes no later than for two models that keep their first-step use throughout.
  \item \emph{The more reliable model grows faster.} Then $g_k$ falls. While it falls slowly enough, there is still one crossover, but later. If it falls faster than \eqref{eq:growth} allows, the less reliable model can regain its advantage on longer tasks, and the two curves can cross more than once.
\end{enumerate}
Condition~\eqref{eq:growth} ensures that, for each resource $k$,
the relative resource efficiency of models $a$ and $b$ can switch
at most once, although a crossover need not occur.

If all three resources exhibit a crossover from model $a$ being
more efficient at shorter task lengths to model $b$ being more
efficient at longer task lengths, the three-regime structure of
Observation~\ref{obs:regimes} is preserved, subject to appropriate treatment of ties at crossover points and the feasibility constraints in~\eqref{eq:floor}.
Unlike in the constant per-step resource model, the
log-resource-ratio curves are no longer necessarily linear or parallel. Consequently, the crossover thresholds, and hence the width of the interval in which decision-maker preferences are
required, depend on the resource-growth and reliability profiles of the competing models.
Condition~\eqref{eq:growth} can be evaluated using measured per-step resource consumption and task-success data, as illustrated in the pilot study in Section~\ref{sec:experiments}.

\paragraph{Failures detected early} If a failed attempt stops at the step that fails, an attempt uses on average $(1-\rho_a^{L})/(1-\rho_a)$ steps, and the resource per successful task becomes $R_{a,k}(L) = r_{a,k} \sum_{j=1}^{L} \rho_a^{-j}$. This still increases with $r_{a,k}$ and decreases with $\rho_a$, so Observation~\ref{obs:prune} holds. The factor that depends on reliability is still the same for all three resources, so there is still one crossover for each resource, in the same order. The crossover lengths are now found numerically, and they are longer than in the basic case, because failed attempts cost less.

\paragraph{From energy to carbon} Carbon emissions equal energy multiplied by the power usage effectiveness of the data centre and the carbon intensity of the local electricity grid. When two models run in the same data centre at the same time, these factors are the same for both and cancel in \eqref{eq:crossover}, so the carbon crossover equals the energy crossover. When the models run in different regions, or at different times of day, the factors enter the starting advantage and move the carbon crossover. This is where the location and timing of inference, studied by carbon-aware routing \citep{sheshanarayana2026gar} and by measurement at provider level \citep{elsworth2025google}, enter our framework.

\paragraph{Uncertainty in measured values} The values of $\rho_a$ and $r_{a,k}$ are estimates. Because the denominator $\ln(\rho_b/\rho_a)$ in \eqref{eq:crossover} is small when two models have similar reliability, small errors in reliability can move a crossover by many steps. We therefore report every crossover length with a confidence interval obtained by resampling the benchmark runs. A wide interval is itself useful: it shows the decision maker that, for tasks of that length, the two models cannot be told apart on the evidence available.

\subsection{Stage 2: choosing within a trade-off\label{sec:stage2}}

When the task-level screen leaves one model, the framework recommends it, and no preferences are needed. When it leaves more than one, the selection question has no answer without knowing what the decision maker values. Stage~2 provides this answer for the efficient set $E(L)$ at the task length $L$.

The values of each objective are rescaled to the range from 0 to 1 within the efficient set,
\begin{equation}
  v_k(a) \;=\; \frac{\max_{b \in E(L)} R_{b,k}(L) - R_{a,k}(L)}{\max_{b \in E(L)} R_{b,k}(L) - \min_{b \in E(L)} R_{b,k}(L)}, \qquad a \in E(L),
\end{equation}
so that the best model on resource $k$ scores 1 and the worst scores 0. If all efficient models use the same amount of a resource, that resource cannot separate them, and we set $v_k(a) = 1$ for every $a$, so that it adds the same amount to every score. Because the scale is set within $E(L)$, a weight states how much the decision maker cares about the \emph{range} of a resource among the efficient models at that task length, not about the resource in general. The range changes with $L$, so the same weights can express different trade-offs at different task lengths. This is one more reason to set weights for a task profile rather than once for all tasks. Weights $w_e, w_c, w_t$, which sum to one, express how much the deployment cares about carbon, cost and time. In the deliberative way, a team sets them once, using AHP pairwise comparisons with the usual consistency check \citep{saaty1977ahp,saaty2008ahp}, or the flexible and interactive trade-off method when the team can give only partial information \citep{dealmeida2016fitradeoff}. In the automated way, they follow from the task profile through a mapping that the organisation can inspect and change. The default rule is SAW,
\begin{equation}
  S(a) \;=\; w_e\, v_e(a) + w_c\, v_c(a) + w_t\, v_t(a), \qquad a \in E(L),
\end{equation}
and TOPSIS \citep{hwang1981madm,ciardiello2023topsissaw} is computed alongside it as a check, on the same rescaled values and with the same weights. When the two methods disagree on the top choice, the framework reports the disagreement. Disagreement does not by itself mean that the models are close. SAW adds up the scores, so a large advantage on one resource can make up for a small loss on another. TOPSIS ranks models by their distance from the best and worst values, and favours models without a very poor value on any resource. The two can therefore rank the same models differently even when the gaps between them are large. A disagreement tells the decision maker that the recommendation depends on how losses on one resource are traded against gains on another, and that this choice should be made with care.

Two features keep Stage~2 small and safe. First, it is needed only inside the bands of Observation~\ref{obs:regimes}. Outside them, the answer does not depend on the weights, so disputes about weights cannot change it. Second, the chance constraint \eqref{eq:floor} is applied before Stage~2. A decision maker who puts all the weight on cost can therefore never push an unreliable model to the top for a task whose actions cannot be undone. This plays the role of the risk veto of \citet{nowak2026framework}, extended from a choice between two models to a choice among many, and derived from the required chance of completing the task rather than set by hand.

\subsection{An illustrative example\label{sec:worked}}

We now bring the parts of the framework together in an example. Table~\ref{tab:example} lists six candidate models with illustrative per-step values. The values are not measurements of named commercial models.

\begin{table}[htbp]
\centering
\caption{Illustrative candidate models. Per-step reliability $\rho$; energy in Wh per step; cost in US cents per step; time in seconds per step; context in thousands of tokens.}
\label{tab:example}
\begin{tabular}{lrrrrr}
\toprule
Model & $\rho$ & Energy & Cost & Time & Context \\
\midrule
A & 0.995 & 0.400 & 3.00 & 2.7 & 400 \\
B & 0.990 & 0.220 & 1.20 & 2.1 & 200 \\
C & 0.970 & 0.060 & 0.25 & 1.5 & 128 \\
\textcolor{gray}{D} & \textcolor{gray}{0.960} & \textcolor{gray}{0.075} & \textcolor{gray}{0.30} & \textcolor{gray}{1.8} & \textcolor{gray}{128} \\
\textcolor{gray}{E} & \textcolor{gray}{0.985} & \textcolor{gray}{0.320} & \textcolor{gray}{1.80} & \textcolor{gray}{3.6} & \textcolor{gray}{1000} \\
F & 0.930 & 0.030 & 0.10 & 0.9 & 32 \\
\bottomrule
\end{tabular}
\end{table}

\paragraph{Stage 1} Model D is worse than model C on every per-step quantity, and model E is worse than model B. By Observation~\ref{obs:prune}, neither D nor E can be efficient for any task length, so both are removed once. Model E would come back only for a task that needs more than 200 thousand tokens of context. Stage~0 would then remove models B, C, D and F, and model A does not dominate E, because E uses less energy and cost per step.

\paragraph{Efficient sets by task length} For tasks that can be retried freely (V1), Table~\ref{tab:frontier} shows the efficient set as task length grows. It changes at the crossover lengths of Observation~\ref{obs:cross}. For example, model F starts with half the energy of model C per step, but C is more reliable, and the energy crossover between them is at about 17 steps.

\begin{table}[htbp]
\centering
\caption{Efficient models per successful task, by task length, for the candidates in Table~\ref{tab:example} (variant V1).}
\label{tab:frontier}
\begin{tabular}{ll}
\toprule
Task length $L$ & Efficient models \\
\midrule
1 to 12 & F \\
13 to 16 & C, F \\
17 to 21 & B, C, F \\
22 to 49 & B, C \\
50 to 76 & A, B, C \\
77 to 181 & A, B \\
182 or more & A \\
\bottomrule
\end{tabular}

\end{table}

\paragraph{Three task profiles} Consider three tasks given to the same set of models.
\begin{itemize}
  \item \emph{A short triage task} with 3 steps, which can be retried freely. Only model F is efficient, so the framework recommends F without asking for weights. Per successful task, F uses about 0.11 Wh against 1.22 Wh for model A, a saving of about 90\%.
  \item \emph{A customer-service task} with 40 steps, which can be retried freely. Models B and C are both efficient. Per successful task, C uses about 8.1 Wh, 34 cents and 203 seconds, and B uses about 13.2 Wh, 72 cents and 126 seconds. A deployment that cares mostly about carbon or cost chooses C, and one that cares mostly about speed chooses B. This is the only one of the three tasks for which Stage~2 is needed. Note that the smallest model, F, is not efficient here: per successful task it uses about 21.9 Wh, almost three times as much as C, although it uses half as much energy per step.
  \item \emph{The same customer-service task, with an irreversible payment step}, allowing at most $K = 2$ attempts with required completion probability $\beta = 0.95$. By \eqref{eq:floor}, each attempt must succeed with probability at least 0.78. At 40 steps only model A meets this, so the framework recommends A. Per successful task, A uses about 19.6 Wh, more than twice as much as C. This is the price of not being able to undo an action.
\end{itemize}

The same six models thus give three different recommendations, and the reason for each one can be read directly from the task profile.

Figure~\ref{fig:parallel} shows why models B and C form a trade-off at 40 steps. On a log scale, the ratio of their resources per successful task is a straight line in task length for each resource, and the three lines are parallel. C is better on time up to about 17 steps, on energy up to about 64 and on cost up to about 77. Between the first and the last crossing, which includes 40 steps, each model is better on some resource.

\begin{figure}[htbp]
\centering
\includegraphics[width=0.75\textwidth]{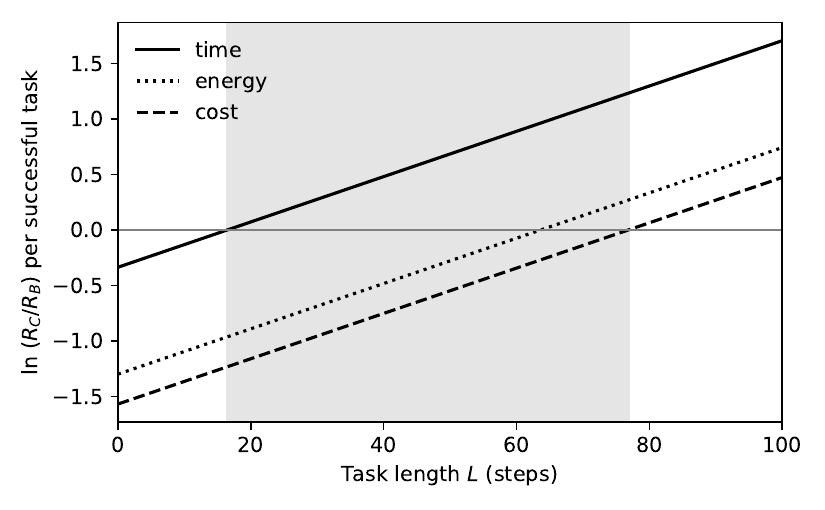}
\caption{Resources per successful task of model C relative to model B, on a log scale, against task length. Below zero, C uses less. Each line crosses zero at the crossover length of its resource. In the shaded band neither model is better on every resource.}
\label{fig:parallel}
\end{figure}

\paragraph{Stage 2 in full} For tasks of 50 to 76 steps, three models are efficient (Table~\ref{tab:frontier}), so this is where Stage~2 has the most work to do. Table~\ref{tab:stage2} shows it for a task of 60 steps that can be retried freely. Panel (a) gives the resources per successful task and the values rescaled within the efficient set. Model C is best on energy and cost, and model A is best on time. Model B is best on nothing, but it is close to the best on every resource. The rescaling also shows the point made in Section~\ref{sec:stage2}: A is only 5\% faster than B, yet A scores 1 on time and B 0.97, while C scores 0 because it takes more than twice as long.

Panel (b) gives the SAW and TOPSIS scores for five sets of weights. With equal weights, B ranks first under both methods, and it stays first when energy or speed gets most of the weight. C ranks first when energy and cost together carry most of the weight and time very little. A, the fastest model, ranks first only when nearly all the weight is on time. Over a grid of all weights in steps of 0.01, SAW puts B first for 66\% of them, C for 33\% and A for less than 1\%. SAW and TOPSIS agree on the top choice for 97\% of the weights. In this example they disagree only near the boundaries where the top choice changes, as in the last row of panel (b), where B and C have almost the same scores. The recommendation, together with these shares, tells the decision maker how much the choice depends on the exact weights.

\begin{table}[htbp]
\centering
\caption{Stage 2 for a task of 60 steps with free retries, for the efficient models A, B and C of Table~\ref{tab:example}. (a) Resources per successful task, and values rescaled within the efficient set (1 is the best). (b) SAW and TOPSIS scores for five sets of weights on energy, cost and time, and the top choice under each method.}
\label{tab:stage2}
\textit{(a) Decision matrix}\\[2pt]
{\small\begin{tabular}{lrrrrrr}
\toprule
Model & Energy (Wh) & Cost (cents) & Time (s) & $v_e$ & $v_c$ & $v_t$ \\
\midrule
A & 32.4 & 243 & 219 & 0.00 & 0.00 & 1.00 \\
B & 24.1 & 132 & 230 & 0.83 & 0.74 & 0.97 \\
C & 22.4 & 93 & 560 & 1.00 & 1.00 & 0.00 \\
\bottomrule
\end{tabular}
}\\[8pt]
\textit{(b) Scores}\\[2pt]
{\small\setlength{\tabcolsep}{4pt}\begin{tabular}{lcrrrcrrrc}
\toprule
Weights & $(w_e, w_c, w_t)$ & \multicolumn{3}{c}{SAW} & & \multicolumn{3}{c}{TOPSIS} & Top choice \\
 & & A & B & C & & A & B & C & SAW / TOPSIS \\
\midrule
Equal & (0.33, 0.33, 0.33) & 0.33 & 0.85 & 0.67 & & 0.41 & 0.83 & 0.59 & B / B \\
Energy first & (0.60, 0.20, 0.20) & 0.20 & 0.84 & 0.80 & & 0.24 & 0.83 & 0.76 & B / B \\
Speed first & (0.20, 0.20, 0.60) & 0.60 & 0.89 & 0.40 & & 0.68 & 0.91 & 0.32 & B / B \\
Energy and cost & (0.45, 0.45, 0.10) & 0.10 & 0.80 & 0.90 & & 0.14 & 0.79 & 0.86 & C / C \\
Energy, little cost & (0.80, 0.04, 0.16) & 0.16 & 0.85 & 0.84 & & 0.17 & 0.83 & 0.83 & B / C \\
\bottomrule
\end{tabular}
}
\end{table}

\subsection{Right-sizing\label{sec:rightsizing}}

The example shows what \emph{right-sizing} means once the selection question is asked per successful task. Right-sizing is not choosing the smallest model. It is choosing, among the models that meet the required chance of completing the task, a model that is efficient in energy, cost and time per successful task, and, within a trade-off band, the one that best matches the stated preferences. Under this definition right-sizing can point to a larger model, as it did for the irreversible task above.

Right-sizing has an environmental benefit that the framework can measure. For each decision, the saving is the difference in energy, and hence in carbon, per successful task between the recommended model and a reference policy, such as always using the most capable model or always using the smallest. Because agentic workloads contain many short, repeated subtasks, the number of decisions where a smaller model is efficient can be large. The framework also shows the opposite case: tasks where a policy of always using small models costs more energy than it saves. Any figure for total savings across a whole service rests on assumptions about the mix of tasks and the carbon intensity of the grid. We therefore report such figures as scenarios with stated bounds, not as measurements.

\section{How can we test the framework?\label{sec:evaluation}}

The framework gives a specific answer to the selection question: it predicts which models are efficient for a task of a given length, and when a smaller model stops being the greener, cheaper or faster choice. This section asks how these predictions can be tested. We first state the claims, then the measurements and the design that test them.

\subsection{Claims to be tested}

The framework rests on four claims, which we state as hypotheses that an unfavourable result would reject.

\begin{description}
  \item[H1 (structure).] On agentic benchmarks, the success curves of models meet the single-crossing condition of Section~\ref{sec:relax}: for pairs of models, the ratio of their success curves falls as tasks get longer, so that their resources per successful task cross at most once. That success falls with length is not itself the test. When success means that every step up to step $L$ is correct, it cannot rise with $L$ by definition. The implication from a decreasing success-probability ratio
to a single crossover holds when the ratio of cumulative
resource consumption between the two models is constant
across task lengths. When this ratio varies, the more general
sufficient condition is given by~\eqref{eq:growth} in
Observation~\ref{obs:growth}, which accounts for both
resource consumption and task-success probabilities.
  \item[H2 (prediction).] Crossover lengths predicted from per-step measurements on one set of tasks agree with the crossovers observed on a separate set of tasks, within their confidence intervals.
  \item[H3 (decision value).] Choices made by the framework have lower regret than choices made by per-call rules. Because no single model need be best on all three resources, regret is measured for each resource separately: it is the extra energy, cost or time per successful task of the chosen model, compared with the model that, in hindsight, uses the least of that resource among the models that meet the chance constraint. A rule has lower regret if it is no worse on any resource and better on at least one. As a single summary we also report the weighted regret, $\sum_k w_k \bigl(1 - v_k(a)\bigr)$, computed with the weights of the task profile on the scale of Section~\ref{sec:stage2}. The per-call rules are: always use the most capable model, always use the smallest model, use the top model of a public leaderboard, a cascade that escalates to a stronger model on failure \citep{chen2024frugalgpt}, and the single-stage weighted rule of \citet{nowak2026framework} applied to all models.
  \item[H4 (savings).] On a substantial share of tasks, the framework selects a smaller model than the most capable one, with no loss in the required chance of completion. On another share of tasks, it selects a larger model than the smallest one and so avoids the energy cost of failed attempts.
\end{description}
H1 is the foundation of the model, and H3 is the result on which the practical value of the framework depends. If H1 fails, the crossover results do not hold as stated. If H3 fails, per-call rules are good enough, and the extra measurement effort is not worth it.

\subsection{What is measured and how}

The design follows one principle: every value is measured on a common test harness, not taken from figures reported by providers, so that the comparison is fair and can be repeated.

\begin{table}[htbp]
\centering
\caption{Measurements used in the evaluation.}
\label{tab:measures}
\begin{tabularx}{\textwidth}{lX}
\toprule
Quantity & How it is measured \\
\midrule
Task success & Automatic check of the final state of each task against the correct state, as in $\tau$-bench \citep{yao2024taubench,barres2025tau2bench}. \\
Per-step reliability & Estimated from complete runs, by the step at which each run first fails. \\
Success curve $S_a(L)$ & Share of successful runs, grouped by the number of steps a task needs. \\
Energy & For open-weight models, measured per inference on known hardware, using published measurements \citep{chung2025mlenergy} and checked against an estimate based on measured speed. For closed models, estimated with stated bounds. \\
Cost & Tokens used, multiplied by the provider's published price. \\
Time & Measured response time for each step, summed over the task. \\
\bottomrule
\end{tabularx}
\end{table}

Each pair of model and task is run several times, so that every measured value carries a confidence interval. Confidence intervals for crossover lengths and for regret are obtained by resampling runs. Tasks are split into two sets: one to estimate per-step values, and one, held out, to test the predictions (H2 and H3). This prevents the framework from being tested on the same data used to set it up.

\subsection{Benchmarks and models}

The benchmarks must vary in task length, so that crossovers can be observed, and must check success automatically. The retail and airline domains of $\tau$-bench meet both needs: tasks differ in the number of tool calls they require, and the airline domain contains actions that cannot be undone, which makes it suitable for variant V2. Function-calling tasks from the Berkeley leaderboard \citep{patil2025bfcl} add a second task family. For the graded variants V3 and V4, a task family with a scored outcome is needed; we use a smaller set of such tasks for illustration.

The main study uses open-weight models of different sizes, served by one provider through one harness. This has three advantages: their energy use can be measured by third parties, their behaviour is not changed without notice, and all models are run under the same conditions. Closed models are added later as a check of how far the results carry over, using public measurements.

\subsection{Robustness}

We check the results in three ways. First, we vary the weights in Stage~2 and report how often the recommendation changes. Second, we vary the measured values within their confidence intervals and report how often the efficient set changes. Third, we test the framework over time: we fix the measurements at one date, and ask whether the framework ranks models released later correctly from their measured per-step values alone.

\section{Experiments and results\label{sec:experiments}}

This section reports a first experiment that tests the framework's answer to the selection question. Before the full study of Section~\ref{sec:evaluation}, we ask a narrower question: with real models on real multi-step tasks, does success fall with task length as the model assumes (H1), and do the crossovers fall where per-step measurements predict (H2)? To answer it cleanly, we use tasks whose length we can set exactly, and models whose energy we can measure rather than estimate. We use two kinds of task and models from several developers, so that the results do not rest on one task or one family of models. The pilot also gives a first, partial view of H3 and H4.

\subsection{Pilot study design}

\paragraph{Tasks} Both tasks are conversations with an agent in which every step has one correct answer that a program can check. The full conversation is kept, so the agent sees its own earlier replies and an error carries forward, as in a real agent. A task of length $L$ succeeds if the first $L$ steps are all correct. Runs are never stopped early, so one run of 150 steps gives the outcome of a task of every length from 1 to 150, together with the energy and time of every step. Every model sees the same 15 tasks of each kind, generated from fixed seeds, and the temperature is set to zero.
\begin{itemize}
  \item \emph{Order task.} The agent keeps the running total of a customer order. At each step it receives a tool result describing one order line, to be added or removed, and must reply with the new total. The arithmetic is deliberately easy, because the tool result states the amount of each line: errors come from keeping track over many steps, not from hard sums.
  \item \emph{Tools task.} The agent serves a customer whose account holds five orders, each with a few items and a delivery city. At each step the customer asks for one change, and the agent must reply with the one correct tool call: change the delivery city, add an item, remove an item, or refund an item. About half of the requests refer back to earlier ones, such as ``add two mugs to that same order'', ``send order F4 to the same city as order B3'' or ``refund every vase in order H9'', which needs the number of vases the order holds now. The tools reply only ``ok'', so the agent must keep track of the account itself. A step is correct if the tool and every argument are right.
\end{itemize}
The two tasks fail in different ways. In the order task an error is a wrong number. In the tools task it is a wrong action or a wrong detail, most often a refund of the wrong quantity or a change that the agent has lost track of.

\paragraph{Models and measurement} We use 15 open-weight models from five developers, from 0.6 to 27 billion parameters (Table~\ref{tab:pilot}): Qwen (Alibaba), Gemma (Google), Llama (Meta), Ministral (Mistral AI) and GLM (Zhipu AI). They are served through Ollama on a single NVIDIA RTX 5060 Ti graphics card (GPU) with 16\,GB of memory. Each model runs alone on the card and fully in its memory, which the test harness checks before every model. Energy is the power of the GPU board, read 20 times a second through the NVIDIA management library and summed over each call. It leaves out the processor, memory and cooling, so it is a lower bound on the energy used. This boundary is the complement of the one used by \citet{panigrahy2026epg}, who measure the processor only; as they stress, the boundary should be stated with every energy figure, since results measured within different boundaries cannot be compared directly. Time is the time per call reported by the server. We do not report money cost, because on one's own hardware it follows from energy and time.

\paragraph{Estimation} Per-step reliability $\rho$ is estimated from the step at which each run first fails, by the maximum-likelihood estimate under a constant per-step rate: the number of steps passed divided by the number of steps passed plus the number of failures. The success curve $S(L)$ is the share of runs whose first failure comes after step $L$. The confidence interval for $\rho$ is the exact (Clopper-Pearson) binomial interval, treating each step passed or failed as one trial. Unlike a resampling interval, it does not shrink to a single point when no run passes a step or no run fails: for a model that fails at the first step of all 15 runs, it still allows a per-step reliability of up to 0.22. Confidence intervals for crossover lengths come from 1000 resamples of runs within each model. When nearly all resamples show no crossover within 150 steps, we report that share instead of an interval.

\paragraph{Reproducibility} The raw logs of every call, the harness and the analysis are in the public repository of the framework. One command recreates every number, table and figure in this section from the committed logs, without a GPU or network access.

\subsection{Results}

\begin{table}[htbp]
\centering
\caption{Pilot results per model and task: estimated per-step reliability, mean GPU energy and time per step, and GPU energy per successfully completed task at three task lengths. $\infty$: no run completed a task of that length. $^\dagger$: removed by the per-step screen, because another model is at least as reliable and uses no more energy or time per step.}
\label{tab:pilot}
\textit{(a) Order task}\\[2pt]
{\small\begin{tabular}{lrrrrrrr}
\toprule
Model & $\hat\rho$ & 95\% CI & J/step & s/step & \multicolumn{3}{c}{kJ per successful task} \\
 & & & & & $L=10$ & $L=20$ & $L=50$ \\
\midrule
Llama 3.2 1B & 0.4643 & 0.2751--0.6613 & 7 & 0.07 & $\infty$ & $\infty$ & $\infty$ \\
Gemma 3 1B$^\dagger$ & 0.0000 & 0.0000--0.2180 & 8 & 0.15 & $\infty$ & $\infty$ & $\infty$ \\
Qwen3 0.6B & 0.9176 & 0.8677--0.9531 & 29 & 0.23 & 0.3 & 1.7 & $\infty$ \\
Qwen3 1.7B & 0.9800 & 0.9673--0.9888 & 30 & 0.22 & 0.1 & 0.5 & 2.4 \\
Llama 3.2 3B & 0.9878 & 0.9797--0.9933 & 30 & 0.21 & 0.1 & 0.4 & 1.8 \\
Ministral 3 3B & 0.9880 & 0.9799--0.9934 & 51 & 0.45 & 0.3 & 0.7 & 3.1 \\
Gemma 3 4B & 0.9882 & 0.9799--0.9937 & 54 & 0.53 & 0.5 & 1.1 & 3.5 \\
Llama 3.1 8B & 0.9905 & 0.9841--0.9948 & 59 & 0.37 & 0.4 & 0.9 & 2.5 \\
GLM-4 9B$^\dagger$ & 0.9444 & 0.9100--0.9686 & 70 & 0.49 & 0.6 & 3.1 & $\infty$ \\
Qwen3 8B & 0.9980 & 0.9948--0.9995 & 91 & 0.56 & 0.6 & 1.3 & 4.0 \\
Ministral 3 8B$^\dagger$ & 0.9969 & 0.9932--0.9989 & 124 & 0.95 & 0.8 & 2.1 & 6.4 \\
Ministral 3 14B & 0.9985 & 0.9957--0.9997 & 136 & 0.94 & 1.0 & 2.1 & 5.8 \\
Qwen3 14B$^\dagger$ & 0.9980 & 0.9948--0.9994 & 153 & 0.90 & 1.1 & 2.4 & 7.0 \\
Gemma 3 12B & 0.9991 & 0.9966--0.9999 & 154 & 1.39 & 1.2 & 2.4 & 9.0 \\
Qwen3.8 27B & 0.9995 & 0.9974--1.0000 & 236 & 1.61 & 2.1 & 4.2 & 11.0 \\
\bottomrule
\end{tabular}
}\\[8pt]
\textit{(b) Tools task}\\[2pt]
{\small\begin{tabular}{lrrrrrrr}
\toprule
Model & $\hat\rho$ & 95\% CI & J/step & s/step & \multicolumn{3}{c}{kJ per successful task} \\
 & & & & & $L=10$ & $L=20$ & $L=50$ \\
\midrule
Gemma 3 1B & 0.7727 & 0.6530--0.8669 & 16 & 0.24 & 1.3 & $\infty$ & $\infty$ \\
Llama 3.2 1B & 0.8404 & 0.7505--0.9078 & 19 & 0.17 & 0.7 & $\infty$ & $\infty$ \\
Qwen3 0.6B$^\dagger$ & 0.8352 & 0.7427--0.9047 & 23 & 0.19 & 0.9 & 6.4 & $\infty$ \\
Qwen3 1.7B & 0.9153 & 0.8641--0.9518 & 23 & 0.18 & 0.2 & 1.5 & $\infty$ \\
Llama 3.2 3B & 0.9296 & 0.8865--0.9601 & 34 & 0.23 & 0.3 & 3.7 & $\infty$ \\
Ministral 3 3B & 0.9513 & 0.9209--0.9725 & 50 & 0.45 & 0.3 & 1.8 & $\infty$ \\
Gemma 3 4B$^\dagger$ & 0.9312 & 0.8891--0.9610 & 51 & 0.50 & 0.7 & 7.0 & $\infty$ \\
Llama 3.1 8B & 0.9533 & 0.9241--0.9736 & 70 & 0.45 & 0.6 & 2.3 & $\infty$ \\
Qwen3 8B$^\dagger$ & 0.9476 & 0.9150--0.9704 & 73 & 0.47 & 0.7 & 3.0 & 48.4 \\
Ministral 3 8B & 0.9704 & 0.9517--0.9833 & 75 & 0.60 & 0.6 & 1.8 & 12.2 \\
GLM-4 9B$^\dagger$ & 0.9470 & 0.9141--0.9700 & 84 & 0.61 & 0.8 & 2.8 & $\infty$ \\
Gemma 3 12B & 0.9710 & 0.9527--0.9837 & 118 & 1.13 & 1.1 & 3.7 & 42.1 \\
Qwen3 14B$^\dagger$ & 0.9689 & 0.9493--0.9825 & 128 & 0.81 & 1.8 & 3.6 & 31.7 \\
Ministral 3 14B & 0.9779 & 0.9638--0.9876 & 129 & 0.91 & 1.0 & 3.0 & 21.8 \\
Qwen3.8 27B & 0.9878 & 0.9797--0.9933 & 199 & 1.47 & 2.0 & 4.2 & 12.2 \\
\bottomrule
\end{tabular}
}
\end{table}

\paragraph{Per-step values and the screen} Reliability rises broadly with size, but far from strictly, and so does energy per step (Table~\ref{tab:pilot}, Figure~\ref{fig:pilot-models}). The tools task is the harder one: the most reliable model makes an error about once in 80 steps there, against once in 2000 steps in the order task. The per-step screen removes four models in the order task and five in the tools task before task length is considered. Some of these removals are not visible from model size or developer. In the order task, Qwen3 14B and Ministral 3 8B are removed by Qwen3 8B, which is at least as reliable and uses less energy and time per step. In the tools task, Qwen3 8B is itself removed, by Llama 3.1 8B and by Ministral 3 3B, a model of less than half its size. Gemma 3 1B fails at the first step of every order task, because it gives wrong totals from the start, and is removed by Llama 3.2 1B. Smaller models do not always use less energy per step either: Gemma 3 1B uses more than Llama 3.2 1B in the order task, and Qwen3 0.6B uses about as much as Qwen3 1.7B in both tasks. So the smallest model is not always the greener choice even per call, and which models are worth considering depends on the task.

\begin{figure}[htbp]
\centering
\begin{subfigure}{0.8\textwidth}
\centering
\includegraphics[width=\textwidth]{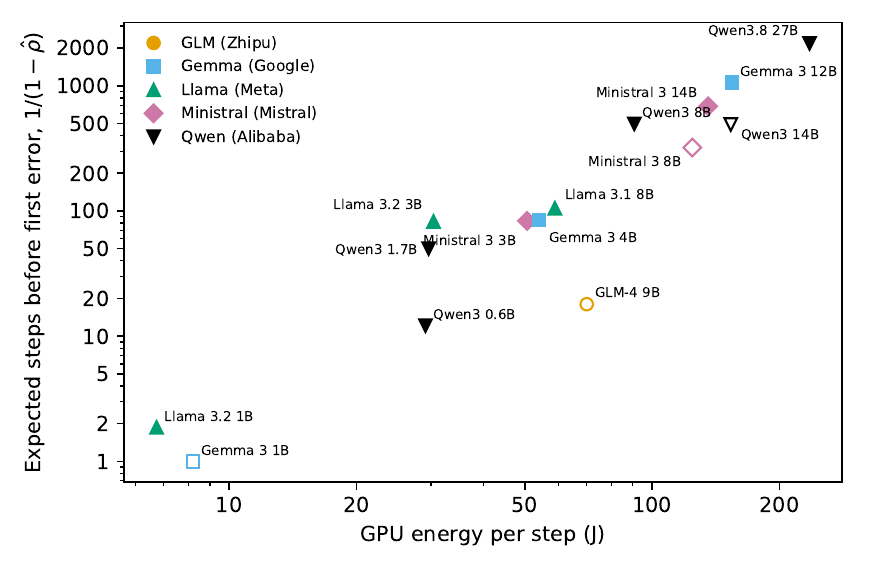}
\caption{Order task}
\end{subfigure}
\begin{subfigure}{0.8\textwidth}
\centering
\includegraphics[width=\textwidth]{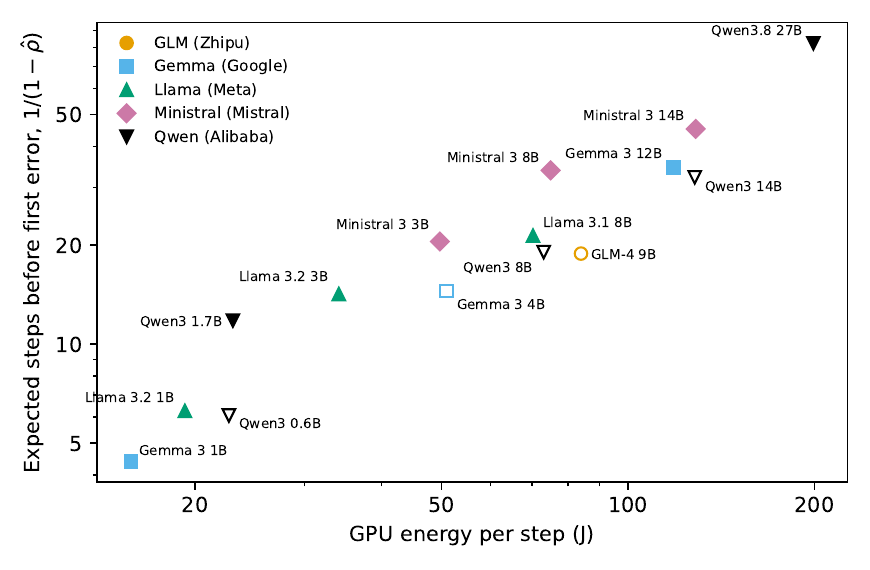}
\caption{Tools task}
\end{subfigure}
\caption{The 15 models on two per-step measures: GPU energy per step, and the expected number of steps before the first error, $1/(1-\hat\rho)$, both on log scales. Colours and markers show the developer. Hollow markers are models removed by the per-step screen. Better models lie towards the top left.}
\label{fig:pilot-models}
\end{figure}

\paragraph{Do the curves cross once? (H1)} Success falls with task length for every model and both tasks (Figure~\ref{fig:pilot}). This is expected, since a task of length $L$ succeeds only if all of its first $L$ steps are correct, and so it is not the test. The test is how the curves of different models compare. A constant per-step rate is only an approximation. For the most reliable models in the order task the success curve stays within 0.1 of the curve that a constant rate implies, but for several mid-sized models it departs by 0.3 to 0.45, because their errors come later in the task than a constant rate predicts. The condition our results need is weaker: that the curves of two models cross at most once for each resource (Section~\ref{sec:relax}). We compare energy and time per successful task for every pair of screened models in which the less reliable model uses less per step: 188 comparisons over the two tasks. In 11, all in the order task, the curves do not cross within 150 steps. In 129 they cross once, and in 48 more than once. If we count only lengths at which both curves rest on at least three successful runs, 33 comparisons still cross more than once. Nearly all of these are of two kinds. Either the two curves stay within about 30\% of each other over the lengths where they cross back and forth, so that the two models are close to a tie and 15 runs cannot separate them, or an extra crossing falls in the first few steps, because the smaller model sometimes fails at the very first step. The single-crossing condition therefore holds well where the choice matters, and fails mainly where it does not.

The energy of a step grows along the task in the pilot, so we also check the sufficient condition \eqref{eq:growth} of Observation~\ref{obs:growth} for the 100 pairs of models compared on energy. Energy per step varies from one step to the next, so we first fit a straight-line trend to each model's mean energy per step. On these trends the energy ratio $g$ rises with task length for 81 pairs (46 of 55 in the order task and 35 of 45 in the tools task): in most pairs the less reliable model, usually the smaller one, grows faster, which is case (ii). Condition \eqref{eq:growth} holds for 98 of the 100 pairs. One exception is Llama 3.2 3B against Ministral 3 3B, whose reliabilities are almost equal (0.9878 and 0.9880), so that $\rho_a/\rho_b$ is so close to one that any fall in $g$ breaks the condition. The other is Gemma 3 12B against Qwen3.8 27B, two of the most reliable models, for which no crossover is observed within 150 steps. Without the trend, using the raw energy of each step from step 10 on, the condition holds for 77 pairs; the other pairs break it at single steps because of noise in the energy of those steps.

\begin{figure}[htbp]
\centering
\begin{subfigure}{0.8\textwidth}
\centering
\includegraphics[width=\textwidth]{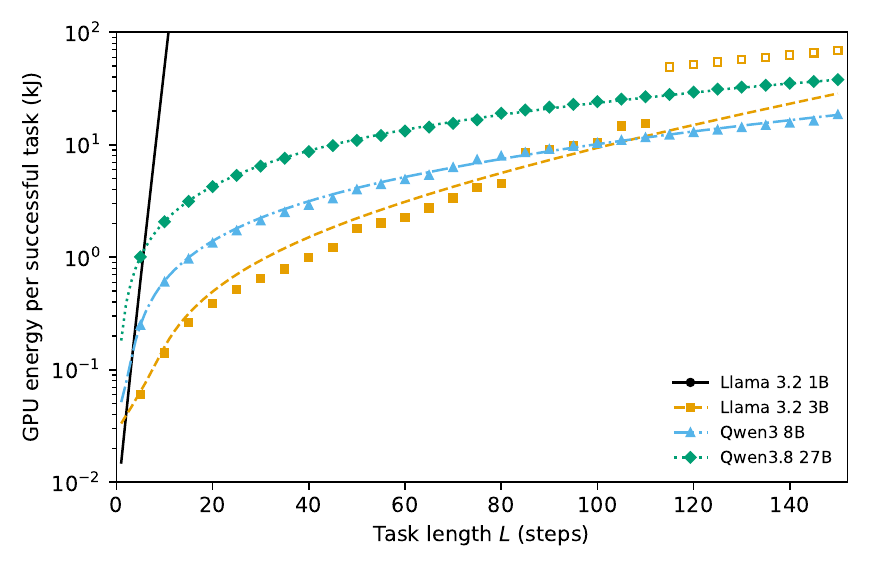}
\caption{Order task}
\end{subfigure}
\begin{subfigure}{0.8\textwidth}
\centering
\includegraphics[width=\textwidth]{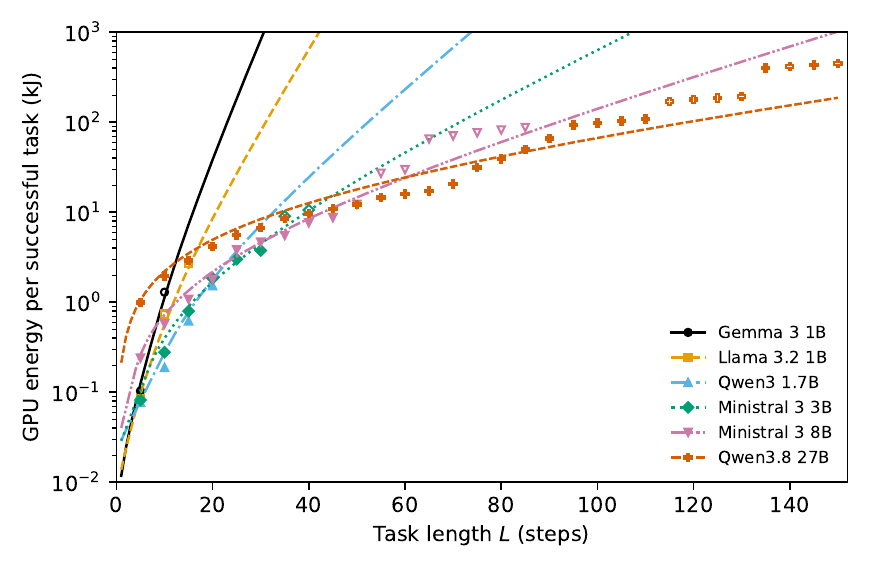}
\caption{Tools task}
\end{subfigure}
\caption{GPU energy per successfully completed task against task length, on a log scale, for the models that use the least energy per completed task at some length, and for the most reliable model. Lines use the estimated per-step reliability and the measured energy of each step; points are the observed values, every five steps, in the colour and marker of their model. Hollow points rest on fewer than three successful runs out of 15, so they are uncertain.}
\label{fig:pilot}
\end{figure}

\paragraph{Do the crossovers fall where predicted? (H2)} Table~\ref{tab:pilot-cross} compares three crossover lengths for energy, for the models shown in Figure~\ref{fig:pilot}. The closed form is the formula of Section~\ref{sec:general}, with the mean energy per step. The fitted value applies the same rule to the measured energy step by step, which allows energy per step to grow as the conversation gets longer, as in Observation~\ref{obs:growth}; in our runs it grows by up to about four times over 150 steps for the smallest models, and by a quarter or less for the largest. The observed value is read from the measured curves.

Over all pairs of screened models, the formula predicts 53 energy crossovers in the order task and 45 in the tools task. In the order task, 8 are predicted beyond 150 steps and none of these is observed. Of the 88 crossovers observed over both tasks, the closed form lies inside the 95\% interval of the observed crossover for 72, and the median difference between the two is 9 steps in the order task and 4 in the tools task. The 16 misses have plain causes. Eight involve a model whose observed curve runs out of successful runs before the predicted crossover: Llama 3.2 3B in the order task, with only one of its 15 runs still free of errors after 111 steps, and Qwen3 1.7B in the tools task, with none after 24 steps. The observed crossover then falls where those runs end, a few steps before the prediction. Four involve two models with almost the same reliability, where the formula is very sensitive to small errors in $\hat\rho$. Two involve Llama 3.1 8B, whose errors come late in the order task, and two are differences of one to three steps for tasks of two to five steps. The full tables for all pairs are in the repository. These predictions are made on the same runs as the observations, so they show that the model describes the data, not yet that it predicts new data. The held-out test of H2 is part of the full study.

\begin{table}[htbp]
\centering
\caption{Energy crossover lengths for the models shown in Figure~\ref{fig:pilot}: the task length beyond which the more reliable model uses less GPU energy per successfully completed task. Closed form: Section~\ref{sec:general} with the mean energy per step. Fitted: measured energy per step and the estimated reliability. Observed: from the measured success curves, with a bootstrap interval. A dash means no crossover within 150 steps. Where some resamples show no crossover, their share is given in brackets; where nearly all show none, only that share is given.}
\label{tab:pilot-cross}
\textit{(a) Order task}\\[2pt]
{\small\begin{tabular}{llrrrrc}
\toprule
Less reliable model & More reliable model & Energy ratio & \multicolumn{4}{c}{Energy crossover $L^\ast$ (steps)} \\
 & & per step & closed form & fitted & observed & observed 95\% CI \\
\midrule
Llama 3.2 1B & Llama 3.2 3B & 0.22 & 2 & 2 & 3 & 1--3 \\
Llama 3.2 1B & Qwen3 8B & 0.07 & 3 & 4 & 5 & 2--5 \\
Llama 3.2 1B & Qwen3.8 27B & 0.03 & 5 & 6 & 5 & 3--5 \\
Llama 3.2 3B & Qwen3 8B & 0.34 & 107 & 108 & 94 & 72--112 \\
Llama 3.2 3B & Qwen3.8 27B & 0.13 & 174 & -- & 112 & 83--112 (28\% none) \\
Qwen3 8B & Qwen3.8 27B & 0.39 & 612 & -- & -- & none ($>$99\%) \\
\bottomrule
\end{tabular}
}\\[8pt]
\textit{(b) Tools task}\\[2pt]
{\small\begin{tabular}{llrrrrc}
\toprule
Less reliable model & More reliable model & Energy ratio & \multicolumn{4}{c}{Energy crossover $L^\ast$ (steps)} \\
 & & per step & closed form & fitted & observed & observed 95\% CI \\
\midrule
Gemma 3 1B & Llama 3.2 1B & 0.82 & 2 & 2 & 1 & 1--15 (18\% none) \\
Gemma 3 1B & Qwen3 1.7B & 0.68 & 2 & 4 & 3 & 2--15 \\
Gemma 3 1B & Ministral 3 3B & 0.32 & 6 & 5 & 4 & 2--10 \\
Gemma 3 1B & Ministral 3 8B & 0.21 & 7 & 9 & 7 & 4--15 \\
Gemma 3 1B & Qwen3.8 27B & 0.08 & 10 & 13 & 15 & 6--15 \\
Llama 3.2 1B & Qwen3 1.7B & 0.84 & 2 & 5 & 5 & 3--20 \\
Llama 3.2 1B & Ministral 3 3B & 0.39 & 8 & 7 & 5 & 3--20 \\
Llama 3.2 1B & Ministral 3 8B & 0.26 & 9 & 12 & 10 & 6--20 \\
Llama 3.2 1B & Qwen3.8 27B & 0.10 & 14 & 17 & 19 & 7--20 \\
Qwen3 1.7B & Ministral 3 3B & 0.46 & 20 & 19 & 24 & 4--25 \\
Qwen3 1.7B & Ministral 3 8B & 0.31 & 20 & 24 & 24 & 12--25 \\
Qwen3 1.7B & Qwen3.8 27B & 0.12 & 28 & 32 & 24 & 17--25 \\
Ministral 3 3B & Ministral 3 8B & 0.66 & 21 & 32 & 31 & 15--44 \\
Ministral 3 3B & Qwen3.8 27B & 0.25 & 37 & 45 & 34 & 18--44 \\
Ministral 3 8B & Qwen3.8 27B & 0.38 & 55 & 61 & 49 & 23--88 \\
\bottomrule
\end{tabular}
}
\end{table}

\paragraph{Which model is greenest depends on the length and the kind of task (H4)} In both tasks the model that uses the least energy per completed task changes as tasks get longer (Figure~\ref{fig:pilot}). In the order task, Llama 3.2 1B is greenest for tasks of one or two steps, Llama 3.2 3B from 3 steps to about 100, and Qwen3 8B beyond. Between about 80 and 100 steps the lead passes back and forth between Llama 3.2 3B, Llama 3.1 8B and Ministral 3 3B, whose curves are close there. Taking energy and time together, the three regimes of Section~\ref{sec:properties} appear between Llama 3.2 3B and Qwen3 8B: the time crossover (96 steps by the formula, 83 observed) comes before the energy crossover (107 and 94), and between them Qwen3 8B is faster while Llama 3.2 3B still uses less energy. In the tools task the greenest model moves up in size much sooner: Llama 3.2 1B up to 4 steps, Qwen3 1.7B up to about 20, Ministral 3 3B up to about 30, Ministral 3 8B up to about 50 to 60, and Qwen3.8 27B beyond. Only six of the 15 models completed any tools task of 50 steps. The stakes are large in both directions. At 20 steps in the order task, Llama 3.2 3B used 0.39\,kJ per completed task, about a tenth of the 4.2\,kJ of the largest model. At 100 steps, Qwen3 1.7B used 39\,kJ, almost four times the 10.4\,kJ of Qwen3 8B. In the order task the largest model is never efficient within 150 steps when tasks can be retried freely, while in the tools task it becomes the greenest choice at about 50 to 60 steps. The same models give different answers for different kinds of task, which is why per-step reliability has to be measured for the task family in question, and not read from a general leaderboard.

\paragraph{Tasks that cannot be undone (variant V2)} The picture changes again when a task must succeed at its first attempt. With a single attempt and a required chance of completion of 0.9, the estimated reliabilities allow order tasks of up to about 9 steps for Llama 3.2 3B, 52 for Qwen3 8B, 72 for Ministral 3 14B, 112 for Gemma 3 12B and 230 for the largest model. From about 112 to 230 steps the largest model is then the only feasible choice, although it is never efficient when retries are free. In the tools task no model meets this requirement beyond about 9 steps. For such tasks the framework's answer is not a model at all: the task must be changed, for example by asking a person to confirm each irreversible change (Section~\ref{sec:discussion}). The task profile, not only the models, decides the answer.

\paragraph{A first look at decision value (H3)} H3 compares the framework with per-call rules. Two of these rules can be tested on the pilot data: always use the largest model, and always use the smallest. For each task length from 1 to \HThreeOrderLengths{} steps, with retries allowed, the framework chooses the model with the least fitted energy per successful task. We compare the observed energy of this choice, and of each rule, with that of the model that used the least in hindsight. In the order task the framework's choice is within 10\% of the best model at \HThreeOrderEnergyFwWithin{} of the \HThreeOrderLengths{} lengths, and its median regret is \HThreeOrderEnergyFwMedian\%. Always using the largest model, \HThreeOrderLargestName, has a median regret of \HThreeOrderEnergyLargeMedian\%: at 20 steps it uses about \HThreeOrderEnergyLargeTimesAtTwenty{} times the energy of the best model. Always using the smallest model, \HThreeOrderSmallestName, has a median regret of \HThreeOrderEnergySmallMedian\% over the lengths at which it completes any task, and it completes no task longer than \HThreeOrderEnergySmallLastCompleted{} steps. In the tools task the framework's choice is within 10\% of the best at \HThreeToolsEnergyFwWithin{} lengths. The largest model does better here, since it becomes the greenest choice at about 50 to 60 steps, but at 10 steps it still uses about \HThreeToolsEnergyLargeTimesAtTen{} times the energy of the best model. The smallest model completes no tools task longer than \HThreeToolsEnergySmallLastCompleted{} steps. Time gives a similar picture, although in the tools task the framework's choice is within 10\% of the best at only \HThreeToolsTimeFwWithin{} lengths. This is a partial test. It covers two of the five per-call rules, and the framework is judged on the same runs that were used to fit it. The full study tests the other three rules on held-out tasks.

\paragraph{What the pilot does not show} The pilot uses two synthetic task families, open-weight models of up to 27 billion parameters, one GPU and 15 runs per model and task. Its energy covers the GPU board only, and its predictions are tested on the same runs that produced them. It tests H3 against only two of the five per-call rules. The full study of Section~\ref{sec:evaluation} addresses each of these limits, using $\tau$-bench tasks, held-out tasks, larger and closed models, and the per-call rules as baselines.

\section{Discussion\label{sec:discussion}}

We return to the selection question: which model should an agent use? Our answer is that it depends on the task in a way that can be predicted. For short tasks, and tasks that can be retried freely, a smaller model is often the efficient choice, and right-sizing saves energy, cost and time. For long tasks, and especially for tasks whose actions cannot be undone, a larger and more reliable model can use less of every resource per completed task. Between these regimes lies a band where the choice depends on what the decision maker values. This section discusses what this answer means for practice, its limits, and the main objections to it.

\subsection{Is smaller greener?}

The view that small models are the sustainable choice for agents \citep{belcak2025slm} is right for many tasks, and our framework supports it where tasks are short. But it is a claim about energy per call. Per completed task, it holds only up to a crossover length that depends on the gap in reliability between models. A policy of always using the smallest model can therefore increase energy use. This complements the measurements of \citet{panigrahy2026epg}, who show that a workflow that looks cheap per call can be costly per successful goal. Our results add the length of the task: the same two models can change places as tasks grow longer. The pilot shows both sides: in the order task, on tasks of 20 steps a 3B model used about a tenth of the energy of a 27B model per completed task, while on tasks of 100 steps a 1.7B model used almost four times as much as an 8B model. In the harder tools task, the 27B model was already the greenest choice from about 50 to 60 steps (Section~\ref{sec:experiments}). For sustainability research, the lesson is that the functional unit is not a detail of reporting. It can reverse the ranking of two systems.

\subsection{What this means for practice}

The framework suggests three things that the designers of agents can do beyond choosing a model. First, shorter tasks favour smaller models, so breaking a long task into short parts that can be checked separately can save resources. Second, for tasks with irreversible actions, the chance constraint can rule out every model (variant V2), as it does in our tools task for tasks of more than about nine steps. The right response is then to change the task, for example by asking a person to confirm before an irreversible step, rather than to look for a different model. Third, a model can be changed during a task: an agent could start with a small model and switch to a larger one when the remaining part of the task passes a crossover length. We return to this in Section~\ref{sec:conclusion}.

\subsection{Limits of the model}

The basic case assumes that steps succeed or fail independently and that any failed step spoils the attempt. Real agents sometimes recover from errors, and errors in one step can make errors in later steps more likely. Section~\ref{sec:relax} shows that the main results survive when the success curve is measured directly rather than assumed, provided the single-crossing condition holds, and this condition is tested as H1. The model also treats all attempts at a task as alike. In practice a second attempt may use what was learned in the first, or be handed to a different model, which is the idea behind cascades. In that case, as \citet{panigrahy2026epg} note, the resource per successful task can depend on the retry limit. The model also covers only tasks whose steps run in sequence. The rules of \citet{yang2026reliableworkflows} for parallel steps and feedback loops suggest how it could be extended to other workflow structures. Finally, the model takes the number of steps as known. In practice it is an estimate, and its uncertainty adds to the uncertainty of the crossover lengths.

\subsection{Measuring energy and carbon}

Energy is the hardest quantity to measure. For open-weight models it can be measured on known hardware; for closed models it must be estimated, and the estimates can be wide \citep{jegham2025hungry}. We keep measured and estimated values apart and carry the ranges of estimates through to the result. Our measures also cover only the energy used to run the model. They leave out the energy used to make the hardware and to train the model, as well as water use. Finally, cheaper and greener inference may lead to more use overall. This rebound effect means that savings per task do not translate directly into lower total energy use, and we do not claim that they do.

\subsection{A changing market}

Models, prices and hardware change every month, so any efficient set is a snapshot. A model removed by the per-step screen can become efficient again if its price falls, if the model that dominated it is withdrawn, or if a new objective is added. A new model, however, can never bring back a model that was dominated before, since it does not change the values of existing models. For this reason the framework keeps every candidate and its dated values, and recomputes the efficient sets when prices, models or measurements change. This is why we see the framework as a continuously maintained decision-support service rather than a one-off study. In the service, a measurement layer keeps the dated per-step values of each candidate, a core engine computes the efficient sets for each task family and each interval of task length, and an interface takes a task profile and returns the recommendation together with the efficient set, the weights used and the source of every value. Agent frameworks can query the service at run time, and people can use it through a web page built on the same interface.

\subsection{Relation to routing}

Routing methods \citep{ong2025routellm,fu2026inflation,li2026progrouter} learn to send each request or step to a suitable model, and often do so well. Our framework does not replace them. It answers a different part of the selection question: which models are worth considering for a task, when the choice needs preferences at all, and how energy enters the choice. Its efficient sets can limit the models a router chooses from, and its per-task measures can be used to evaluate routers in the same terms of energy, cost and time.

\section{Conclusion and future directions\label{sec:conclusion}}

This paper asked which language model an AI agent should use, so that its tasks are completed reliably with the least energy, cost and time. We argued that this selection question must be answered per successfully completed task, not per call. On this basis we built a general decision model and showed that it has a simple structure. Models that are worse on every per-step measure can be removed for all tasks. Between any two remaining models, each resource has one crossover task length, and the crossovers divide tasks into three regimes, only one of which needs the preferences of the decision maker. The same model covers binary and graded outcomes, and tasks whose actions can or cannot be undone. The resulting two-stage framework separates questions of fact, which models are efficient, from questions of value, which trade-off to accept. It also gives right-sizing a precise meaning, which sometimes points to a smaller model and sometimes to a larger one.

Several directions follow. First, the full study set out in Section~\ref{sec:evaluation} will test whether the structure seen in the pilot of Section~\ref{sec:experiments} holds on agentic benchmarks, with held-out tasks and per-call rules as baselines. Second, the model can be extended to switching models during a task, which turns the single choice into a policy over the course of the task. Third, agents that work together in teams raise the question of how crossovers combine across agents. Fourth, a study with practitioners can test whether the trade-offs that the framework shows match the choices that the designers of agents actually make, and how their preferences can best be elicited. Fifth, the measure of resources can be widened to include the energy of hardware and training, and water use. Finally, the framework will be released as an open service with its measurement harness and data, so that its efficient sets can be updated as the market changes.

\bibliographystyle{elsarticle-harv}
\bibliography{references}

\end{document}